\documentclass{article}

\usepackage{amsmath}
\usepackage{amssymb}
\usepackage{amsthm}
\usepackage{algorithm}
\usepackage{algorithmic}
\usepackage{fullpage}
\usepackage{tikz}
\usepackage{tabu}
\usepackage{graphics}
\usepackage{subfig}
\usepackage{epstopdf}
\usepackage{epsfig}
\usepackage{caption}
\usepackage{float}
\numberwithin{equation}{section} 
\newtheorem{theorem}{Theorem}[section]                   
\newtheorem{lemma}[theorem]{Lemma}         
\newtheorem{corollary}[theorem]{Corollary}     

\newtheorem{claim}[theorem]{Claim}
\newtheorem{ass}[theorem]{Assumption}

\newtheorem{definition}[theorem]{Definition}
\newtheorem{remark}[theorem]{Remark}
   
\newcommand{\new}[1]{{#1}}

\def\R{{\mathbb{R}}}

\def\lin{\text{lin}}
\newcommand{\norm}[1]{\|#1\|}

\usepackage{enumitem}
\usepackage{cite}

\begin{document}
\title{Fast Algorithms for Demixing \\ Sparse Signals from Nonlinear Observations}
\author{Mohammadreza Soltani and Chinmay Hegde \\
Electrical and Computer Engineering Department \\
Iowa State University\thanks{
This work was supported in part by the National Science Foundation under the grant CCF-1566281. Parts of this work also appear in an Iowa State University technical report~\cite{SoltaniHegde} and a conference paper to be presented in the 2016 Asilomar Conference in {November 2016~\cite{SoltaniHegde_Asilomar}.}
}
}
\date{}
\maketitle	

\begin{abstract}
We study the problem of \emph{demixing} a pair of sparse signals from nonlinear observations of their 
superposition. Mathematically, we consider a nonlinear signal observation model, $y_i = g(a_i^Tx) + e_i, \ i=1,\ldots,m$, where $x = \Phi w+\Psi z$ denotes the superposition signal, $\Phi$ and $\Psi$ are orthonormal bases in $\mathbb{R}^n$, and $w, z\in\mathbb{R}^n$ are sparse coefficient vectors of the constituent signals. Further, we assume that the observations are corrupted by a subgaussian additive noise. Within this model, $g$ represents a nonlinear \emph{link} function, and $a_i\in\mathbb{R}^n$ is the $i$-th row of the measurement matrix, $A\in\mathbb{R}^{m\times n}$. Problems of this nature arise in several applications ranging from astronomy, computer vision, and machine learning.

In this paper, we make some concrete algorithmic progress for the above demixing problem. Specifically, we consider two scenarios: (i) the case when the demixing procedure has no knowledge of the link function, and (ii) the case when the demixing algorithm has perfect knowledge of the link function. In both cases, we provide fast algorithms for recovery of the constituents $w$ and $z$ from the observations. Moreover, we support these algorithms with a rigorous theoretical analysis, and derive (nearly) tight upper bounds on the sample complexity of the proposed algorithms for achieving stable recovery of the component signals. Our analysis also shows that the running time of our algorithms is essentially as good as the best possible.

We also provide a range of numerical simulations to illustrate the performance of the proposed algorithms on both real and synthetic signals and images. Our simulations show the superior performance of our algorithms compared to existing methods for demixing signals and images based on convex optimization. In particular, our proposed methods yield demonstrably better sample complexities as well as improved running times, thereby enabling their applicability to large-scale problems.

\end{abstract}


%

\section{Introduction}
\label{sec::intro}

\subsection{Setup}

In numerous signal processing applications, the problem of \emph{demixing} is of special interest. In simple terms, demixing involves disentangling two (or more) constituent signals from observations of their linear superposition. Formally, consider a discrete-time signal $x \in \R^n$ that can be expressed as the superposition of two signals: 
\[
x=\Phi w + \Psi z \, ,
\] 
where $\Phi$ and $\Psi$ are orthonormal bases of $\mathbb{R}^n$, and $w, z\in\mathbb{R}^n$ are the corresponding basis coefficients. The goal of signal demixing, in this context, is to reliably recover the constituent signals (equivalently, their basis representations $w$ and $z$) from the superposition signal $x$. 

Demixing suffers from a fundamental \emph{identifiability} issue since the number of unknowns ($2n$) is greater than the number of observations ($n$). This is easy to see: suppose for simplicity that $\Phi = \Psi = I_n$, the canonical basis of $\R^n$, and therefore, $x = w + z$. Now, suppose that both $w$ and $z$ have only one nonzero entry in the first coordinate. Then, there is an {infinite} number of $w$ and $z$ that are consistent with the observations $x$, and any hope of recovering the true components is lost. Therefore, for the demixing problem to have an identifiable solution, one inevitably has to assume some type of \emph{incoherence} between the constituent signals (or more specifically, between the corresponding bases $\Phi$ and $\Psi$)~\cite{elad2005simultaneous,donoho2006stable}. Such an incoherence assumption certifies that the components are sufficiently ``distinct" and that the recovery problem is well-posed. Please see Section~\ref{Sec::Perm} for a formal definition of incoherence. 

However, even if we assume that the signal components are sufficiently incoherent, demixing poses additional challenges under stringent observation models. Suppose, now, that we only have access to {undersampled linear} measurements of the signal, i.e., we record:
\begin{equation}
\label{eq:obs}
y= Ax \, ,
 \end{equation}
where $A\in\mathbb{R}^{m\times n}$ denotes the measurement operator and where $m < n$. In this scenario, the demixing problem is further confounded by the fact that $A$ possesses a nontrivial null space. In this case, it might seem impossible to recover the components $x$ and $z$ since $A$ possesses a nontrivial null space. Once again, this problem is highly ill-posed and further structural assumptions on the constituent signals are necessary. Under-determined problems of this kind have recently received significant attention in signal processing, machine learning, and high-dimensional statistics. In particular, the emergent field of \emph{compressive sensing}~\cite{CandesCS,DonohoCS,foucart2013} shows that it is indeed possible to exactly reconstruct the underlying signals under certain assumptions on $x$, provided the measurement operator is designed carefully. This intuition has enabled the design of a wide range of efficient architectures for signal acquisition and processing \cite{cscamera,MisEld::2009::From-theory}.

In this paper, we address an \emph{even} more challenging question in the demixing context. Mathematically, we consider a \emph{noisy, nonlinear} signal observation model, formulated as follows:
\begin{align}\label{MainModel}
y_i = g( \langle a_i, \Phi w+\Psi z \rangle) + e_i, \ i=1,\ldots,m \, .
\end{align}
Here, as before, the superposition signal is modeled as $x  = \Phi w+\Psi z$. Each observation is generated by the composition of a linear functional of the signal $\langle a_i, x \rangle$, with a (scalar) nonlinear function $g$. Here, $g$ is sometimes called a \textit{link} or \textit{transfer} function, and $a_i$ denotes the $i^{\textrm{th}}$ row of a linear measurement matrix $A \in \R^{m \times n}$. For full generality, in~\eqref{MainModel} we assume that each observation $y_i$ is corrupted by additive noise; the noiseless case is realized by setting $e_i = 0$. We will exclusively consider the ``measurement-poor" regime where the number of observations $m$ is much smaller than the ambient dimension $n$.

For all the reasons detailed above, the problem of recovering the coefficient vectors $w$ and $z$ from the measurements $y$ seems daunting. Therefore, we make some structural assumptions. Particularly, we assume that $w$ and $z$ are $s$-\emph{sparse} (i.e., they contain no more than $s$ nonzero entries). Further, we will assume perfect knowledge of the bases $\Phi$ and $\Psi$, and the measurement matrix $A$. The noise vector $e \in \R^m$ is assumed to be stochastic, zero mean, and bounded. 
Under these assumptions, we will see that it is indeed possible to stably recover the coefficient vectors, with a number of observations that is proportional to the sparsity level $s$, as opposed to the ambient dimension $n$.  

The nonlinear link function $g$ plays a crucial role in our algorithm development and analysis. In signal processing applications, such nonlinearities may arise due to imperfections caused during a measurement process, or inherent limitations of the measurement system, or due to quantization or calibration errors. We discuss such practical implications more in detail below. On an abstract level, we consider two distinct scenarios. In the first scenario, the link function may be non-smooth, non-invertible, or even unknown to the recovery procedure. This is the more challenging case, but we will show that recovery of the components is possible even without knowledge of $g$. 
In the second scenario; the link function is a known, smooth, and strictly monotonic function. This is the somewhat simpler case, and we will see that this leads to significant improvements in recovery performance both in terms of theory and practice.


\subsection{Our Contributions} 

In this paper, we make some concrete algorithmic progress in the demixing problem under nonlinear observations. In particular, we study the following scenarios depending on certain additional assumptions made on~\eqref{MainModel}:

\begin{enumerate}[leftmargin=*]

\item \textbf{Unknown $g$}. We first consider the (arguably, more general) scenario where the nonlinear link function $g$  may be non-smooth, non-invertible, or even unknown. In this setting, we do not explicitly model the additive noise term in \eqref{MainModel}. For such settings, we introduce a novel demixing algorithm that is non-iterative, does not require explicit knowledge of the link function $g$, and produces an estimate of the signal components. We call this algorithm \textsc{OneShot} to emphasize its non-iterative nature. It is assumed that \textsc{OneShot} possess oracle knowledge of the measurement matrix $A$, and orthonormal bases $\Phi$ and $\Psi$.

We supplement our proposed algorithm with a rigorous theoretical analysis and derive upper bounds on the \emph{sample complexity} of demixing with nonlinear observations. In particular, we prove that the sample complexity of \textsc{OneShot} to achieve an estimation error $\kappa$ is given by $m = \mathcal{O}(\frac{1}{\kappa^2} s \log \frac{n}{s})$ provided that the entries of the measurement matrix are i.i.d.\ standard normal random variables. 

\item \textbf{Known $g$}. Next, we consider the case where the nonlinear link function $g$ is known, smooth, and monotonic. In this setting, the additive noise term in \eqref{MainModel} is assumed to be bounded either absolutely, or with high probability. For such (arguably, milder) settings, we provide an iterative algorithm for demixing of the constituent signals in \eqref{MainModel} given the nonlinear observations $y$. We call this algorithm \textsc{Demixing with Hard Thresholding}, or \textsc{DHT} for short. In addition to knowledge of $g$, we assume that \textsc{DHT} possesses oracle knowledge of $A$, $\Phi$, and $\Psi$.

Within this scenario, we also analyze two special sub-cases: 

\noindent
{\textbf{Case 2a: Isotropic measurements}}. We assume that the measurement vectors $a_i$ are independent, isotropic random vectors that are {incoherent} with the bases $\Phi$ and $\Psi$. This assumption is more general than the i.i.d.\ standard normal assumption on the measurement matrix made in the first scenario, and is applicable to a wider range of measurement models. For this case, we show that the sample complexity of \textsc{DHT} is upper-bounded by $m=\mathcal{O}(s~\text{polylog }n)$, independent of the estimation error $\kappa$. 

\noindent {\textbf{Case 2b: Subgaussian measurements.}} we assume that the rows of the matrix $A$ are independent subgaussian isotropic random vectors. This is also a generalization of the i.i.d.\ standard normal assumption made above, but more restrictive than Case 2a. In this setting, we obtain somewhat better sample complexity. More precisely, we show that the sample complexity of \textsc{DHT} is $m =\mathcal{O}(s\log\frac{n}{s})$ for sample complexity, matching the best known sample complexity bounds for recovering a superposition of $s$-sparse signals from \emph{linear} observations~\cite{spinisit,spinIT}.

\end{enumerate}

In both the above cases, the underlying assumption is that the bases $\Phi$ and $\Psi$ are sufficiently incoherent, and that the sparsity level $s$ is small relative to the ambient dimension $n$. In this regime, we show that \textsc{DHT} exhibits a \emph{linear} rate of convergence, and therefore the computational complexity of \textsc{DHT} is only a logarithmic factor higher than \textsc{OneShot}.  Table~\ref{AlgoProp} provides a summary of the above contributions for the specific case where $\Phi$ is the identity (canonical) basis and $\Psi$ is the discrete cosine transform (DCT) basis, and places them in the context of the existing literature on some nonlinear recovery methods~\cite{plan2014high,hassibigeneralized,plangeneralized}. We stress that these previous works do not explicitly consider the demixing problem, but in principle the algorithms of \cite{plan2014high,hassibigeneralized,plangeneralized} can be extended to the demixing setting as well.

\renewcommand{\arraystretch}{1.5}
\begin{table}
\centering
\caption{\emph{Summary of our contributions, and comparison with existing methods for the concrete case where $\Phi$ is the identity and $\Psi$ is the DCT basis. Here, $s$ denotes the sparsity level of the components, $n$ denotes the ambient dimension, $m$ denotes the number of samples, and $\kappa$ denotes estimation error.}}
\begin{tabular}{|c|c|c|c|c|}
\hline
Algorithms & Sample complexity & Running time & Measurements & Link function\\
 \hline
 \textsc{LASSO} \cite{plan2014high} & $\mathcal{O}(\frac{s}{\kappa^2} \log \frac{n}{s})$ & $\text{poly}(n)$ & Gaussian & unknown \\ \hline
 \textsc{OneShot} & $\mathcal{O}(\frac{s}{\kappa^2} \log \frac{n}{s})$ & $\mathcal{O}(mn)$ & Gaussian & unknown\\ \hline
 \textsc{DHT} & $\mathcal{O}(s~\textrm{polylog}~n)$ &  $\mathcal{O}(mn \log \frac{1}{\kappa})$ & Isotropic rows & known\\ \hline
 \textsc{DHT} & $\mathcal{O}(s \log \frac{n}{s})$ &  $\mathcal{O}(mn \log \frac{1}{\kappa})$ & Subgaussian & known\\ 
\hline
\end{tabular}
\label{AlgoProp}
\end{table}

\subsection{Techniques}

At a high level, our recovery algorithms are based on the now-classical method of \emph{greedy} iterative thresholding. In both methods, the idea is to first form a {proxy} of the signal components, followed by hard thresholding to promote sparsity of the final estimates of the coefficient vectors $w$ and $z$. The key distinguishing factor from existing methods is that the greedy thresholding procedures used to estimate $w$ and $z$ are \emph{deliberately myopic}, in the sense that each thresholding step operates \emph{as if the other component did not exist at all}. Despite this apparent shortcoming, we are still able to derive bounds on recovery performance when the signal components are sufficiently incoherent.

Our first algorithm, \textsc{OneShot}, is based on the recent, pioneering approach of \cite{plan2014high}, which describes a simple (but effective) method to estimate a high-dimensional signal from unknown nonlinear observations. Our first main contribution of this paper is to extend this idea to the nonlinear demixing problem, and to precisely characterize the role of incoherence in the recovery process. Indeed, a variation of the approach of \cite{plan2014high} (described in Section \ref{Sec::Res}) can be used to solve the nonlinear demixing problem as stated above, with a similar two-step method of first forming a proxy, and then performing a convex estimation procedure (such as the LASSO~\cite{lasso}) to produce the final signal estimates. However, as we show below in our analysis and experiments, \textsc{OneShot} offers superior performance to this approach. The analysis of \textsc{OneShot} is based on a geometric argument, and leverages the \emph{Gaussian mean width} for the set of sparse vectors, which is a statistical measure of complexity of a set of points in a given space. 

While \textsc{OneShot} is simple and effective, one can potentially do much better if the link function $g$ were available at the time of recovery. Our second algorithm, \textsc{DHT}, leverages precisely this intuition. First, we formulate our nonlinear demixing problem in terms of an {optimization problem} with respect to a specially-defined loss function that depends on the nonlinearity $g$. Next, for solving the proposed optimization problem, we propose an iterative method to solve the optimization problem, up to an additive approximation factor. Each iteration with \textsc{DHT} involves a proxy calculation formed by computing the gradient of the loss function, followed by (myopic) projection onto the constraint sets. Again, somewhat interestingly, this method can be shown to be \emph{linearly convergent}, and therefore only incurs a small (logarithmic) overhead in terms of running time. The analysis of \textsc{DHT} is based on bounding certain parameters of the loss function known as the restricted strong convexity (RSC) and restricted strong smoothness (RSS) constants.\footnote{Quantifying algorithm performance by bounding RSC and RSC constants of a given loss function are quite widespread in the machine learning literature~\cite{negahban2009unified,bahmani2013greedy,yuan2013gradient,jain2014iterative}, but have not studied in the context of signal demixing.}

Finally, we provide a wide range of simulations to verify empirically our claims both on synthetic and real data. We first compare the performance of \textsc{OneShot} with the convex optimization method of~\cite{plan2014high} for nonlinear demixing via a series of phase transition diagrams. Our simulation results show that \textsc{OneShot} outperforms this convex method significantly in both demixing efficiency as well as running time, and consequently makes it an attractive choice in large-scale problems. However, as discussed below, the absence of knowledge of the link function induces an inevitable scale ambiguity in the final estimation\footnote{Indeed, following the discussion in~\cite{plan2014high}, any demixing algorithm that does not leverage knowledge of $g$ is susceptible to such a scale ambiguity.}. For situations where we know the link function precisely, our simulation results show that \textsc{DHT} offers much better statistical performance compared to \textsc{OneShot}, and is even able to recover the scale of the signal components explicitly. We also provide simulation results on real-world natural images and astronomical data to demonstrate robustness of our approaches.

\subsection{Organization}

The rest of this paper is organized as follows. Section~\ref{Sec::rel} describes several potential applications of our proposed approach, and relationship with prior art. Section~\ref{Sec::Perm} introduces some key notions that are used throughout the paper. Section~\ref{Sec::Alg} contains our proposed algorithms, accompanied by  analysis of their performance; complete proofs are deferred to Section~\ref{Sec::proof}. Section~\ref{Sec::Res} lists the results of a series of numerical experiments on both synthetic and real data, and Section~\ref{Sec::Con} provides concluding remarks.

\section{Applications and Related Work}
\label{Sec::rel}

Demixing problems of various flavors have been long studied in research areas spanning signal processing, statistics, and physics, and we only present a small subset of relevant related work. In particular, demixing methods have been the focus of significant research over the fifteen years, dating back at least to~\cite{DonohoBP}. The work of Elad et al.~\cite{elad2005simultaneous} and Bobin et al.~\cite{bobin2007morphological} posed the demixing problem as an instance of \emph{morphological components analysis} (MCA), and formalized the observation model \eqref{eq:obs}. Specifically, these approaches posed the recovery problem in terms of a convex optimization procedure, such as the LASSO~\cite{lasso}. The work of Pope et al.~\cite{studer2012recovery} analyzed somewhat more general conditions under which stable demixing could be achieved.

More recently, the work of~\cite{mccoyTropp2014} showed a curious phase transition behavior in the performance of the convex optimization methods. Specifically, they demonstrated a sharp statistical characterization of the achievable and non-achievable parameters for which successful demixing of the signal components can be achieved. Moreover, they extended the demixing problem to a large variety of signal structures beyond sparsity via the use of general \emph{atomic norms} in place of the $\ell_1$-norm in the above optimization. See~\cite{mccoy2014convexity} for an in-depth discussion of atomic norms, their statistical and geometric properties, and their applications to demixing.

Approaches for (linear) demixing has also considered a variety of signal models beyond sparsity. The {robust PCA} problem~\cite{candes2011rpca,Venkat2009sparse,Venkat2011rank} involves the separation of low-rank and sparse matrices from their sum. This idea has been used in several applications ranging from video surveillance to sensor network monitoring. In machine learning applications, the separation of low-rank and sparse matrices has been used for latent variable model selection~\cite{Venkat2010latent} as well as the robust alignment of multiple occluded images \cite{peng2012rasl}. Another type of signal model is the \emph{low-dimensional manifold} model. In \cite{spinisit,spinIT}, the authors proposed a greedy iterative method for demixing signals, arising from a mixture of known low-dimensional manifolds by iterative projections onto the component manifolds.

The problem of signal demixing from linear measurements belongs to a class of linear inverse problems that underpin \emph{compressive sensing}~\cite{CandesCS,DonohoCS}; see~\cite{foucart2013} for an excellent introduction. There, the overarching goal is to recover signals from (possibly randomized) linear measurements of the form~\eqref{eq:obs}. More recently, it has been shown that compressive sensing techniques can also be extended to inverse problems where the available observations are manifestly \emph{nonlinear}. For instance, in $1$-bit compressive sensing~\cite{boufounos20081,planRomanLin} the linear measurements of a given signal are quantized in the extreme fashion such that the measurements are binary ($\pm1$) and only comprise the sign of the linear observation. Therefore, the amplitude of the signal is completely discarded by the quantization operator.
Another class of such nonlinear recovery techniques can be applied to the classical signal processing problem of phase retrieval~\cite{candes2015phase} which is somewhat more challenging than $1$-bit compressive sensing. In this problem, the phase information of the signal measurements may be irrecovably lost and we have only access to the amplitude information of the signal \cite{candes2015phase}. Therefore, the recovery task here is to retrieve the phase information of the signal from random observations. Other related works include approaches for recovering low-rank matrices from nonlinear observations~\cite{davenport20141,ganti2015matrix}. 
We mention in passing that inverse problems involving nonlinear observations have also long been studied in the statistical learning theory literature; see~\cite{kalai2009,kakade2011,ganti2015learning,Cons2015NIPS} for recent work in this area. Analogous to our scenarios above, these works consider both known as well as unknown link functions; these two classes of approaches are respectively dubbed as Generalized Linear Models (GLM) learning methods and Single Index Model (SIM) learning methods. 

For our algorithmic development, we build upon a recent line of efficient, iterative methods for signal estimation in high dimensions~\cite{beck2013sparsity,bahmani2013greedy,yuan2013gradient,plan2014high,jain2014iterative,yang2015sparse}. The basic idea is to pose the recovery as a (non-convex) optimization problem in which an objective function is minimized over the set of $s$-sparse vectors. Essentially, these algorithms are based on well-known iterative thresholding methods proposed in the context of sparse recovery and compressive sensing~\cite{blumensath2009iterative,needell2009cosamp}. The analysis of these methods heavily depends on the assumption that the objective function satisfies certain (restricted) regularity conditions; see Sections~\ref{Sec::Perm} and~\ref{Sec::proof} for details. Crucially, we adopt the approach of~\cite{negahban2009unified}, which introduces the concept of the restricted strong convexity (RSC) and restricted strong smoothness (RSS) constants of a loss function. Bounding these constants in terms of problem parameters $n$ and $s$, as well as the level of incoherence in the components, enables explicit characterization of both sample complexity and convergence rates.



\section{Preliminaries}
\label{Sec::Perm}

In this section, we introduce some notation and key definitions. Throughout this paper, $\|.\|_p$ denotes the $\ell_p$-norm of a vector in $\mathbb{R}^n$, and $\|A\|$ denotes the spectral norm of the matrix $A\in\mathbb{R}^{m\times n}$. Let $\Phi$ and $\Psi$ be orthonormal bases of $\R^n$. Define the set of sparse vectors in the bases $\Phi$ and $\Psi$ as follows:
\begin{align*}
K_1 &= \{\Phi a\ | \ \|a\|_0\leq s_1\}, \\ 
K_2 &= \{\Psi a\ | \ \|a\|_0\leq s_2\},
\end{align*}
and define $K = \{a\ | \ \|a\|_0\leq s\}.$ We use $B_2^n$ to denote the unit $\ell_2$ ball. Whenever we use the notation $t=[w;z]$, the vector $t$ is comprised by stacking column vectors $w$ and $z$.

In order to bound the sample complexity of our proposed algorithms, we will need some concepts from high-dimensional geometry. First, we define a statistical measure of complexity of a set of signals, following~\cite{plan2014high}.

\begin{definition}(Local gaussian mean width.) \label{def 3.3}
For a given set $K \in\mathbb{R}^n$, the local gaussian mean width (or simply, local mean width) is defined as follows $\forall~t>0$:
$$W_t(K) = \mathbb{E}\sup_{x, y\in K, \|x-y\|_2 \leq t}\langle g, x-y\rangle.$$
where $g\sim\mathcal{N}(0,I_{n \times n})$.
\end{definition}

Next, we define the notion of a \emph{polar norm} with respect to a given subset $Q$ of the signal space:
\begin{definition}(Polar norm.)
For a given $x\in \mathbb{R}^n$ and a subset of $Q\in \mathbb{R}^n$, the polar norm with respect to $Q$ is defined as follows:
$$\|x\|_{Q^o} = \sup_{u\in Q}\langle x, u\rangle.$$
\end{definition} 
Furthermore, for a given subset of $Q\in \mathbb{R}^n$, we define $Q_t = (Q-Q)\cap t B_2^n$. Since $Q_t$ is a symmetric set, one can show that the polar norm with respect to $Q_t$ defines a semi-norm. Next, we use the following standard notions from random matrix theory~\cite{vershynin2010introduction}:

\begin{definition}(Subgaussian random variable.) \label{subgau}
A random variable $X$ is called subgaussian if it satisfies the following:
\begin{align*}
\mathbb{E}\exp\left(\frac{c X^2}{\|X\|_{\psi_2}^2}\right)\leq 2,
\end{align*}
where $c >0$ is an absolute constant and $\|X\|_{\psi_2}$ denotes the $\psi_2$-norm which is defined as follows:
\begin{align*}
\|X\|_{\psi_2}=\sup_{p\geq 1}\frac{1}{\sqrt{p}}(\mathbb{E}|X|^p)^{\frac{1}{p}}.
\end{align*}
\end{definition}

\begin{definition}(Isotropic random vectors.)
A random vector-valued variable $v\in\mathbb{R}^n$ is said to be isotropic if  $\mathbb{E} vv^T=I_{n\times n}$.
\end{definition}

In order to analyze the computational aspects of our proposed algorithms (in particular, \textsc{DHT}), we will need the following definition from~\cite{negahban2009unified}:
\begin{definition}\label{rssrsc}
A loss function $f$ satisfies \textit{Restricted Strong Convexity/Smoothness (RSC/RSS)} if:
\begin{align*}
m_{4s}\leq\|\nabla^2_{\xi} f(t)\|\leq M_{4s},
\end{align*}
where $\xi = \mathrm{supp}(t_1)\cup \mathrm{supp}(t_2)$, for all $\|t_i\|_0\leq 2s$ and $i=1,2$.
Also, $m_{4s}$ and $M_{4s}$ are (respectively) called the RSC and RSS constants. Here $\nabla^2_{\xi} f(t)$ denotes a $4s\times 4s$ sub-matrix of the Hessian matrix, $\nabla^2 f(t)$, comprised of row/column indices in $\xi$.
\end{definition} 

As discussed earlier, the underlying assumption in all demixing problems of the form \eqref{MainModell} is that the constituent bases are sufficiently \textit{incoherent} as per the following definition:
\begin{definition}($\varepsilon$-incoherence.)\label{incoherence}
The orthonormal bases $\Phi$ and $\Psi$ are said to be $\varepsilon$-incoherent if:  
\begin{align}
\varepsilon = \sup_{\substack{\|u\|_0\leq s,\ \|v\|_0\leq s  \\ \|u\|_2 = 1,\ \|v\|_2 = 1}}|\langle{\Phi u, \Psi v}\rangle|.
\end{align}
\end{definition}

The parameter $\varepsilon$ is related to the so-called {mutual coherence} parameter of a matrix. Indeed, if we consider the (overcomplete) dictionary $\Gamma = [\Phi \, \Psi]$, then the mutual coherence of $\Gamma$ is given by $\gamma = \max_{i \neq j} | (\Gamma^T \Gamma)_{ij} |$. Moreover, one can show that $\varepsilon \leq s \gamma$  \cite{foucart2013}.  

We now formally establish our \emph{signal model}. Consider a signal $x\in \mathbb{R}^n$ that is the superposition of a pair of sparse vectors in different bases, i.e.,
\begin{equation}\label{superpositoin}
x = \Phi w + \Psi z \, ,
\end{equation}
where $\Phi, \Psi\in \mathbb{R}^{n\times n}$ are orthonormal bases, and $w, z\in \mathbb{R}^n$ such that $\|w\|_0\leq s$, and  $\|z\|_0\leq s$.
We define the following quantities:
\begin{align}\label{sp}
\bar{x} = \frac{\Phi \bar{w} + \Psi \bar{z}}{\|\Phi \bar{w} + \Psi \bar{z}\|}_2 = \alpha(\Phi \bar{w} + \Psi \bar{z}) ,
\end{align}
where 
$\alpha = \frac{1}{\|\Phi \bar{w} + \Psi \bar{z}\|_2},~\bar{w} = \frac{w}{\|w\|_2},~\bar{z} = \frac{z}{\|z\|_2}.$ Also, define the coefficient vector, $t =  [w;z ]\in\mathbb{R}^{2n}$. as the vector obtaining by stacking the individual coefficient vectors $w$ and $z$ of the component signals.

We now state our \emph{measurement model}. Consider the nonlinear observation model:
\begin{align}\label{MainModell}
y_i = g(a_i^Tx) + e_i, \ i=1\dots m,
\end{align}
where $x \in \mathbb{R}^n$ is the superposition signal given in~\eqref{superpositoin}, and $g : \mathbb{R} \mapsto \mathbb{R}$ represents a nonlinear link function. 
We denote $g(x)$ as the derivative of $\Theta(x)$, i.e., $\Theta'(x) = g(x)$. As mentioned above, depending on the knowledge of the link function $g$, we consider two scenarios:

\begin{enumerate}[leftmargin=*]
\item 
In the first scenario, the nonlinear link function may be non-smooth, non-invertible, or even unknown. In this setting, we assume the noiseless observation model, i.e., $y = g(Ax)$. In addition, we assume that the measurement matrix is populated by i.i.d.\ unit normal random variables.

\item
In this setup, $g$ represents a known nonlinear, differentiable, and strictly monotonic function. Further, in this scenario, we assume that the observation $y_i$ is corrupted by a subgaussian additive noise with $\|e_i\|_{\psi_2}\leq\tau$ for $i=1, \ldots,m$. We also assume that the additive noise has zero mean and independent from $a_i$, i.e., $\mathbb{E}\left(e_i\right) = 0$ for $i=1, \ldots, m$. In addition, we assume that the measurement matrix consists of either (2a) isotropic random vectors that are incoherent with $\Phi$ and $\Psi$, or (2b) populated with subgaussian random variables.

\end{enumerate}

We highlight some additional clarifications for the second case. In particular, we make the following :
\begin{ass}
\label{ass_g}
There exist nonnegative $l_1, l_2 >0 \ (\text{resp., nonpositive parameters } l_1, l_2<0)$ such that  $0<l_1\leq\ g^{\prime}(x)\leq l_2 \ (\text{resp.}~l_1\leq\ g^{\prime}(x)\leq l_2<0)$.
\end{ass} 
In words, the derivative of the link function is strictly bounded either within a positive interval or within a negative interval. In this paper, we focus on the case when $0<l_1\leq\ g^{\prime}(x)\leq l_2$. The analysis of the complementary case is similar. 

The lower bound on $g^\prime(x)$ guarantees that the function $g$ is a monotonic function, i.e., if $x_1 < x_2$  then $g(x_1) < g(x_2)$. Moreover, the upper bound on $g^\prime(x)$ guarantees that the function $g$ is \textit{Lipschitz} with constant $l_2$. 
Such assumptions are common in the nonlinear recovery literature~\cite{negahban2009unified, yang2015sparse}.
\footnote{Using the monotonicity property of $g$ that arises from Assumption~\ref{ass_g}, one might be tempted to  simply apply the inverse of the link function on the measurements $y_i$ in~\eqref{MainModell} convert the nonlinear demixing problem to the more amenable case of linear demixing, and then use any algorithm (e.g.,~\cite{spinIT}) for recovery of the constituent signals. However, this na\"{i}ve way could result in a large error in the estimation of the components, particularly in the presence of the noise $e_i$ in~\eqref{MainModell}. This issue has been also considered in~\cite{yang2015sparse} for generic nonlinear recovery both from a theoretical as well as empirical standpoint.}  


In Case 2a, the vectors $a_i$ (i.e., the rows of $A$) are independent isotropic random vectors. For this case, in addition to incoherence between the component bases, we also need to define a measure of \emph{cross}-coherence between the measurement matrix $A$ and the dictionary $\Gamma$. 
The following notion of cross-coherence was introduced in the early literature of compressive sensing~\cite{candes2007sparsity}:

\begin{definition}(Cross-coherence.)\label{mutualcoherence}
The cross-coherence parameter between the measurement matrix $A$ and the dictionary $\Gamma=[\Phi \ \Psi]$ is defined as follows:
\begin{align}
\vartheta = \max_{i,j}\frac{a_i^T\Gamma_j}{\|a_i\|_2},
\end{align}
where $a_i$ and $\Gamma_j$ denote the $i^{\textrm{th}}$ row of the measurement matrix $A$ and the $j^{\textrm{th}}$ column of the dictionary $\Gamma$.
\end{definition}

The cross-coherence assumption implies that $\big{\|}a_i^T\Gamma_{\xi}\big{\|}_{\infty}\leq\vartheta$ for $i=1,\ldots,m$, where $\Gamma_{\xi}$ denotes the restriction of the columns of the dictionary to set $\xi \subseteq [2n]$, with $|\xi|\leq 4s$ such that $2s$ columns are selected from each basis $\Phi$ and $\Psi$.

\section{Algorithms and Theoretical Results}
\label{Sec::Alg}

Having defined the above quantities, we now present our main results. As per the previous section, we study two distinct scenarios:

\subsection{When the link function $g$ is unknown}
Recall that we wish to recover components $w$ and $z$ given the nonlinear measurements $y$ and the matrix $A$. Here and below, for simplicity we assume that the sparsity levels $s_1$ and $s_2$, specifying the sets $K_1$ and $K_2$, are equal, i.e., $s_1 = s_2 = s$. 
The algorithm (and analysis) effortlessly extends to the case of unequal sparsity levels.
Our proposed algorithm, that we call \textsc{OneShot}, is described in pseudocode form below as Algorithm~\ref{alg:oneshot}.

The mechanism of \textsc{OneShot} is simple, and \emph{deliberately} myopic. At a high level, \textsc{OneShot} first constructs a \emph{linear estimator} of the target superposition signal, denoted by $\widehat{x}_\text{\lin} = \frac{1}{m} A^T y$. Then, it performs independent projections of $\widehat{x}_\text{\lin}$ onto the constraint sets $K_1$ and $K_2$. Finally, it combines these two projections to obtain the final estimate of the target superposition signal.

\begin{algorithm}[!t]
\caption{\label{alg:oneshot}
\textsc{OneShot}}
\textbf{Inputs:} Basis matrices $\Phi$ and $\Psi$, measurement matrix $A$, measurements $y$, sparsity level $s$. \\
\textbf{Outputs:} Estimates  $\widehat{x}=\Phi\widehat{w} + \Psi\widehat{z}$, $\widehat{w}\in K_1$, $\widehat{z}\in K_2$

\setlength{\parskip}{1em}
$\widehat{x}_\text{\lin}\leftarrow\frac{1}{m}A^{T}y$\qquad~\{form linear estimator\}\\
$b_1\leftarrow\Phi^*\widehat{x}_\lin$\qquad~~~\{forming first proxy\}\\
$\widehat{w}\leftarrow\mathcal{P}_s(b_1)$\qquad~~~\{sparse projection\}\\
$b_2\leftarrow\Psi^*\widehat{x}_\lin$\qquad~~~\{forming second proxy\}\\
$\widehat{z}\leftarrow\mathcal{P}_s(b_2)$\qquad~~~~\{sparse projection\}\\
$\widehat{x}\leftarrow\Phi\widehat{w} + \Psi\widehat{z}$\quad~~~\{Estimating $\widehat{x}$\}
\end{algorithm}

In the above description of \textsc{OneShot}, we have used the following \emph{projection} operators:
$$\widehat{w} = \mathcal{P}_s(\Phi^*\widehat{x}_\text{\lin}),  \quad \hat{z} = \mathcal{P}_s(\Psi^*\widehat{x}_\text{\lin}).$$
Here, $\mathcal{P}_s$ denotes the projection onto the set of (canonical) $s$-sparse signals $K$ and can be implemented by hard thresholding, i.e., any procedure that retains the $s$ largest coefficients of a vector (in terms of absolute value) and sets the others to zero\footnote{The typical way is to sort the coefficients by magnitude and retain the $s$ largest entries, but other methods such as randomized selection can also be used.}. Ties between coefficients are broken arbitrarily. Observe that  \textsc{OneShot} is \emph{not} an iterative algorithm, and this in fact enables us to achieve a fast running time. 

We now provide a rigorous performance analysis of \textsc{OneShot}. Our proofs follow the geometric approach provided in~\cite{plan2014high}, specialized to the demixing problem. In particular, we derive an upper bound on the estimation error of the \emph{component} signals $w$ and $z$, modulo scaling factors.
In our proofs, we use the following result from~\cite{plan2014high}, restated here for completeness.
\begin{lemma}(Quality of linear estimator).\label{lemma 4.1}
Given the model in Equation \eqref{superpositoin}, the linear estimator, $\widehat{x}_\text{\lin}$, is an unbiased estimator of $\bar{x}$ (defined in~\eqref{sp}) up to constants. That is, $\mathbb{E}(\widehat{x}_\textrm{lin})= \mu\bar{x}$ and:
$\mathbb{E}\|\widehat{x}_\textrm{lin} - \mu\bar{x}\|_2^2 = \frac{1}{m}[\sigma^2 +\eta^2(n-1)],$
where 
$\mu = \mathbb{E}(y_1\langle a_1, \bar{x}\rangle),~\sigma^2 = Var(y_1\langle a_1, \bar{x}\rangle),~\eta^2 = \mathbb{E}(y_1^2).
$
\end{lemma}

We now state our first main theoretical result, with the full proof provided below in Section \ref{Sec::proof}.

\begin{theorem}
\label{thm:main}
Let $y\in\mathbb{R}^m$ be the set of measurements \new{generated using a nonlinear function $g$ that satisfies the conditions of Lemma (4.9) in~\cite{plan2014high}}\footnote{Based on this lemma, the nonlinear function $g$ is odd, nondecreasing, and sub-multiplicative on $\mathbb{R}^+$.}. Let $A\in \mathbb{R}^{m\times n}$ be a random matrix with i.i.d.\ standard normal entries. Also, let $\Phi, \Psi\in \mathbb{R}^{n\times n}$ are bases with $\varepsilon \leq 0.65$, where $\varepsilon$ is as defined in Def.\ \ref{incoherence}. If we use \textsc{Oneshot} to recover estimates of $w$ and $z$ (modulo a scaling) described in equations \eqref{superpositoin} and \eqref{sp}, then the estimation error for $w$ (similarly, $z$) satisfies the following upper bound in expectation $\forall {\rho} >0$: 
\begin{align}\label{MT}
\mathbb{E}\|\widehat{w}-\mu\alpha{w}\|_2\leq {\rho} + \frac{2}{\sqrt{m}}\left(4\sigma + \eta\frac{W_{\rho}(K)}{\rho}\right) +8\mu\varepsilon \, .
\end{align}
\end{theorem}
The constant $0.65$ is chosen for convenience and can be strengthened. The authors of~\cite{planRmanrobust,plan2014high} provide upper bounds on the local mean width $W_{\rho}(K)$ of the set of $s$-sparse vectors. In particular, for any ${\rho} > 0$ they show that $W_{\rho}(K) \leq C {\rho} \sqrt{s \log (2n/s)}$ for some absolute constant $C$. By plugging in this bound and letting ${\rho}\rightarrow 0$, we can combine components $\widehat{w}$ and $\widehat{z}$ which gives the following:
\begin{corollary}
\label{corr:estx}
With the same assumptions as Theorem~\ref{thm:main}, the error of nonlinear estimation incurred by the final output $\widehat{x}$ satisfies the upper bound:
\begin{align}\label{eq:estx}
\mathbb{E}\|\widehat{x}-\mu\bar{x}\|_2\leq\frac{4}{\sqrt{m}}\left(4\sigma + C\eta\sqrt{s \log (2n/s)}\right) +16\mu\varepsilon.
\end{align}
\end{corollary}

\begin{corollary}(Example quantitative result). The constants $\sigma, \eta, \mu$ depend on the nature of the nonlinear function $f$, and are often rather mild. For example, if $f(x) = \mathrm{sign}(x)$, then we may substitute
$$\mu = \sqrt{\frac{2}{\pi}}\approx 0.8,\qquad\sigma^2 = 1- \frac{2}{\pi}\approx 0.6,\qquad\eta^2=1,$$
in the above statement.
Hence, the bound in \eqref{eq:estx} becomes:
\begin{align}
\mathbb{E}\|\widehat{x}-\mu\bar{x}\|_2\leq\frac{4}{\sqrt{m}}\left(3.1 + C\sqrt{s \log (2n/s)}\right) +13\varepsilon \, .
\end{align}
\end{corollary}

\begin{proof}
Using Lemma \ref{lemma 4.1}, $\mu = \mathbb{E}(y_i\langle a_i, \bar{x}\rangle)$ where $y_i = \mathrm{sign}(\langle a_i, x\rangle)$. Since $a_i\sim\mathcal{N}(0,I)$ and $\bar{x}$ has unit norm, $\langle a_i, \bar{x}\rangle\sim\mathcal{N}(0,1)$. Thus, $\mu = \mathbb{E}|g| = \sqrt{\frac{2}{\pi}}$ where $g\sim\mathcal{N}(0,I)$. Moreover, we can write $\sigma^2 = \mathbb{E}(|g|^2)-\mu^2 = 1- \frac{2}{\pi}$. Here, we have used the fact that $|g|^2$ obeys the $\chi^2_1$ distribution with mean 1. Finally, $\eta^2 = \mathbb{E}(y_1^2) = 1$.
\end{proof}

In contrast with demixing algorithms for traditional (linear) observation models, our estimated signal $\widehat{x}$ outputting from \textsc{OneShot} can differ from the true signal $x$ by a scale factor. Next, suppose we fix $\kappa > 0$ as a small constant, and suppose that the incoherence parameter $\varepsilon = c\kappa$ for some constant $c$, and that the number of measurements scales as: 
\begin{equation}
\label{eq:oneshot_samples}
m = \mathcal{O}\left(\frac{s}{\kappa^2}\log\frac{n}{s}\right).
\end{equation}
Then, the (expected) estimation error $\| \widehat{x} - \mu \bar{x} \| \leq O(\kappa)$. In other words, the \emph{sample complexity} of \textsc{OneShot} is given by $m = \mathcal{O}(\frac{1}{\kappa^2} s \log(n/s))$, which resembles results for the linear observation case~\cite{spinIT,plan2014high}\footnote{Here, we use the term ``sample-complexity" as the number of measurements required by a given algorithm to achieve an estimation error $\kappa$. However, we must mention that algorithms for the linear observation model are able to achieve stronger sample complexity bounds that are independent of $\kappa$.}.

We observe that the estimation error in \eqref{eq:estx} is upper-bounded by $\mathcal{O}(\varepsilon)$. This  is meaningful only when $\varepsilon \ll 1$, or when $s \gamma \ll 1$. Per the Welch Bound~\cite{foucart2013}, the mutual coherence $\gamma$ satisfies $\gamma \geq 1/\sqrt{n}$. Therefore, Theorem \ref{thm:main} provides non-trivial results only when $s = o(\sqrt{n})$. This is consistent with the \emph{square-root bottleneck} that is often observed in demixing problems; see~\cite{tropp07} for detailed discussions.

The above theorem obtains a bound on the expected value of the estimation error. We can derive a similar upper bound that holds with high probability. In this theorem, we assume that the \emph{measurements} $y_i$ for $i=1,2,\ldots,m$ have a \emph{sub-gaussian} distribution (according to Def.~\ref{subgau}). We obtain the following result, with full proof deferred to Section \ref{Sec::proof}.

\begin{theorem}(High-probability version of Thm.\ \ref{thm:main}.)
\label{thm:high}
Let $y\in\mathcal{R}^m$ be a set of measurements with a sub-gaussian distribution. Assume that $A\in \mathbb{R}^{m\times n}$ is a random matrix with $i.i.d$ standard normal entries. Also, assume that $\Phi, \Psi\in \mathbb{R}^{n\times n}$ are two bases with incoherence $\varepsilon\leq 0.65$ as in Definition \ref{incoherence}. Let $0\leq s'\leq\sqrt{m}$. If we use \textsc{Oneshot} to recover $w$ and $z$ (up to a scaling) described in \eqref{superpositoin} and \eqref{sp}, then the estimation error of the output of \textsc{Oneshot} satisfies the following:

\begin{align} 
\|\widehat{x} - \mu\bar{x}\|_2 \leq \frac{4\eta}{\sqrt{m}}\left(3s'+C'\sqrt{s\log\frac{2n}{s}}\right) + 16\mu\varepsilon,
\end{align}
with probability at least $1-4\exp(-\frac{cs'^2\eta^4}{\|y_1\|_{\psi_2}^4})$ where $C', c> 0$ are absolute constants. The coefficients $\mu, \sigma$, and $\eta$ are given in Lemma \ref{lemma 4.1}. Here, $\|y_1\|_{\psi_2}$ denotes the $\psi_2$-norm of the first measurement $y_1$ (Definition~\ref{subgau}). 

\end{theorem}

In Theorem~\ref{thm:high}, we stated the tail probability bound of the estimation error for the superposition signal, $x$. Similar to Theorem~\ref{thm:main}, we can derive a completely analogous tail probability bound in terms of the constituent signals $w$ and $z$.

\subsection{When the link function $g$ is known}\label{SubDHT}

The advantages of \textsc{OneShot} is that it enables fast demixing, and can handle even unknown, non-differentiable link functions. But its primary weakness is that the sparse components are recovered only up to an arbitrary scale factor. This can lead to high estimation errors in practice, and this can be unsatisfactory in applications. Moreover, even for reliable recovery up to a scale factor, its sample complexity is inversely dependent on the estimation error. To solve these problems, we propose a different, iterative algorithm for recovering the signal components. Here, the main difference is that the algorithm is assumed to possess (perfect) knowledge of the nonlinear link function, $g$. 

Recall that we define $\Gamma = [\Phi \ \Psi]$ and $t = [w; z]\in\mathbb{R}^{2n}$. First, we formulate our demixing problem as the minimization of a special loss function $F(t)$:
\begin{equation} \label{optprob}
\begin{aligned}
& \underset{t \in \mathbb{R}^{2n}}{\text{min}}
\ \ F(t) = \frac{1}{m}\sum_{i=1}^m \Theta(a_i^T\Gamma t) - y_i a_i^T\Gamma t \\
& \text{s.\ t.}  \quad \|t\|_0\leq 2s.
\end{aligned}
\end{equation}

Observe that the loss function $F(t)$ is \emph{not} the typical squared-error function commonly encountered in statistics and signal processing applications. In contrast, it heavily depends on the nonlinear link function $g$ (via its integral $\Theta$). Instead, such loss functions are usually used in GLM and SIM estimation in the statistics literature~\cite{negahban2009unified}. 
In fact, the objective function in~\eqref{optprob} can be considered as the \emph{sample} version of the problem: 
$$\min_{t \in \mathbb{R}^{2n}} \  \mathbb{E}(\Theta(a^T\Gamma t) - y a^T\Gamma t),$$ 
where $a, y$ and $\Gamma$ satisfies the model~\eqref{MainModell}. It is not hard to show that the solution of this problem satisfies $\mathbb{E}(y_i|a_i) = g(a_i^T\Gamma t)$. 
We note that the gradient of the loss function can be calculated in closed form:
\begin{align}\label{GradientofOb}
\nabla F(t) &= \frac{1}{m}\sum_{i=1}^{m}\Gamma^T a_i g(a_i^T\Gamma t) - y_i\Gamma^Ta_i, \\
& = \frac{1}{m} \Gamma^T A^T (g(A \Gamma t) - y) \nonumber.
\end{align}

We now propose an \emph{iterative} algorithm for solving~\eqref{optprob} that we call it \textsc{Demixing with Hard Thresholding (DHT)}. The method is detailed in Algorithm \ref{algDHT}.
At a high level, \textsc{DHT} iteratively refines its estimates of the constituent signals $w, z$ (and the superposition signal $x$). At any given iteration, it constructs the gradient using~\eqref{GradientofOb}. Next, it updates the current estimate according to the gradient update being determined in Algorithm \ref{algDHT}. Then, it performs hard thresholding using the operator $\mathcal{P}_{2s}$ to obtain the new estimate of the components $w$ and $z$. This procedure is repeated until a stopping criterion is met. See Section~\ref{Sec::Res} for the choice of stopping criterion and other details.
We mention that the initialization step in Algorithm~\ref{algDHT} is arbitrary and can be implemented (for example) by running \textsc{OneShot} and obtaining initial points $\left(x^0, w^0, z^0\right)$. We use this initialization in our simulation results.

\begin{algorithm}[t]
\caption{Demixing with Hard Thresholding (DHT)
\label{algDHT}}
\begin{algorithmic}
\STATE \textbf{Inputs:} Bases $\Phi$ and $\Psi$, measurement matrix $A$, link function $g$, measurements $y$, sparsity level $s$, step size $\eta'$. 
\STATE \textbf{Outputs:} Estimates  $\widehat{x}=\Phi\widehat{w} + \Psi\widehat{z}$, $\widehat{w}$, $\widehat{z}$
\STATE\textbf{Initialization:}
\STATE$\left(x^0, w^0, z^0\right)\leftarrow\textsc{arbitrary initialization}$
\STATE$k \leftarrow 0$
\WHILE{$k\leq N$}
\STATE $t^k \leftarrow [ w^k ; z^k ]$\quad\quad\quad\quad\quad~\{forming constituent vector\}
\STATE $t_1^k\leftarrow\frac{1}{m}\Phi^TA^T(g(Ax^k) - y)$ 
\STATE$t_2^k\leftarrow\frac{1}{m}\Psi^TA^T(g(Ax^k) - y)$
\STATE$\nabla F^k \leftarrow [ t_1^k ; t_2^k ]$
\quad\quad\quad\quad~\{forming gradient\}
\STATE${\tilde{t}}^k = t^k - \eta'\nabla F^k$
\quad\quad\quad~~\{gradient update\}
\STATE$[ w^k ; z^k ]\leftarrow\mathcal{P}_{2s}\left(\tilde{t}^k\right)$  
\quad\quad~~\{sparse projection\}
\STATE$x^k\leftarrow\Phi w^k + \Psi z^k$\quad\quad\quad~~\{estimating $\widehat{x}$\}
\STATE$k\leftarrow k+1$
\ENDWHILE
\STATE\textbf{Return:} $\left(\widehat{w}, \widehat{z}\right)\leftarrow \left(w^N, z^N\right)$
\end{algorithmic}
\end{algorithm}

Implicitly, we have again assumed that both component vectors $w$ and $z$ are $s$-sparse; however, as above we mention that Algorithm~\ref{algDHT} and the corresponding analysis easily extend to differing levels of sparsity in the two components. In Algorithm \ref{algDHT}, $\mathcal{P}_{2s}$ denotes the projection of vector $\tilde{t}^k\in\mathbb{R}^{2n}$ on the set of $2s$ sparse vectors, again implemented via hard thresholding.

We now provide our second main theoretical result, supporting the convergence analysis of \textsc{DHT}. In particular, we derive an upper bound on the estimation error of the constituent vector $t$ (and therefore, the component signals $w,z$). The proofs of Theorems~\ref{mainThConvergence},~\ref{SampleComplexityNonormal} and~\ref{SampleComplexitysubg} are deferred to section~\ref{Sec::proof}.

\begin{theorem}(Performance of \textsc{DHT})
\label{mainThConvergence}
Consider the measurement model~\eqref{MainModell} with all the assumptions mentioned for the second scenario in Section~\ref{Sec::Perm}. Suppose that the corresponding objective function $F$ satisfies the RSS/RSC properties with constants $M_{6s}$ and $m_{6s}$ on the set $J_k$ with $|J_k|\leq 6s$ ($k$ denotes the $k^{th}$ iteration) such that $1\leq\frac{M_{6s}}{m_{6s}}\leq\frac{2}{\sqrt{3}}$. Choose a step size parameter $\eta'$ with $\frac{0.5}{M_{6s}}<\eta^{\prime}<\frac{1.5}{m_{6s}}.$
Then, \textsc{DHT} outputs a sequence of estimates $t^k = [w^k; z^k]$ that satisfies the following upper bound (in expectation) for $k\geq 1$: 
\begin{align}
\label{eq:linconverge}
\|t^{k+1} - t^*\|_2\leq\left(2q\right)^k\|t^0-t^*\|_2 + C\tau\sqrt{\frac{s}{m}}, 
\end{align}
where $q = \sqrt{1+{\eta^{\prime}}^2M_{6s}^2-2\eta^{\prime} m_{6s}}$  and $C>0$ is a constant that  depends on the step size $\eta^{\prime}$ and the convergence rate $q$. 
Also, $t^{*} = [w ;z]$ where $w$ and $z$ are the true (unknown) vectors in model~\eqref{MainModel}. 
\end{theorem}

Equation \eqref{eq:linconverge} indicates that Algorithm~\ref{algDHT} (\textsc{DHT}) enjoys a linear rate of convergence. In particular, for the noiseless case $\tau = 0$, this implies that Alg.\ \ref{algDHT} returns a solution with accuracy $\kappa$ after $N = \mathcal{O}(\log \frac{\norm{t^0 - t}_2}{\kappa})$ iterations. The proof of Theorem~\ref{mainThConvergence} leverages the fact that the objective function $F(t)$ in~\eqref{optprob} satisfies the RSC/RSS conditions specified in Definition~\ref{rssrsc}. Please refer to Section~\ref{Sec::proof} for a more detailed discussion.
Moreover, we observe that in contrast with \textsc{OneShot}, \textsc{DHT} can recover the components $w$ and $z$ without any ambiguity in scaling factor, as depicted in the bound~\eqref{eq:linconverge}. We also verify this observation empirically in our simulation results in Section~\ref{Sec::Res}.

Echoing our discussion in Section~\ref{Sec::Perm}, we consider two different models for the measurement matrix $A$ and derive upper bounds on the sample complexity of \textsc{DHT} corresponding to each case. First, we present the sample complexity of Alg.\ \ref{algDHT} when the measurements are chosen to be isotropic random vectors, corresponding to Case (2a) described in the introduction:

\begin{theorem}(Sample complexity when the rows of $A$ are isotropic.)
\label{SampleComplexityNonormal}
Suppose that the rows of $A$ are independent isotropic random vectors.
In order to achieve the requisite RSS/RSC properties of Theorem~\ref{mainThConvergence}, the number of samples needs to scale as:
\[
m=\mathcal{O}(s\log n\log^2s\log(s\log n)),
\]
provided that the bases $\Phi$ and $\Psi$ are incoherent enough.
\end{theorem}

The sample complexity mentioned in Theorem~\ref{SampleComplexityNonormal} incurs an extra (possibly parasitic) poly-logarithmic factor relative to the sample complexity of \textsc{OneShot}, stated in \eqref{eq:oneshot_samples}. However, the drawback of \textsc{OneShot} is that the sample complexity depends inversely on the estimation error $\kappa$, and therefore a very small target error would incur a high overhead in terms of number of samples. 

Removing all the extra logarithmic factors remains an open problem in general (although some improvements can be obtained using the method of~\cite{fourier_samples}). However, if we assume  additional structure in the measurement matrix $A$, we can decrease the sample complexity even further. This corresponds to Case 2b.

\begin{theorem}(Sample complexity when the elements of $A$ are subgaussian.)
\label{SampleComplexitysubg}
Assume that all assumptions and definitions in Theorem~\ref{mainThConvergence} holds except that the rows of matrix $A$ are independent subgaussian isotropic random vectors. Then, in order to achieve the requisite RSS/RSC properties of Theorem~\ref{mainThConvergence}, the number of samples needs to scale as:
\[
m=\mathcal{O}\left(s\log \frac{n}{s} \right),
\]
provided that the bases $\Phi$ and $\Psi$ are incoherent enough.
\end{theorem}

The leading big-Oh constant in the expression for $m$ in Theorems~\ref{SampleComplexityNonormal} and~\ref{SampleComplexitysubg} is somewhat complicated, and hides the dependence on the incoherence parameter $\varepsilon$, the mutual coherence $\vartheta$, the RSC/RSS constants, and the growth parameters of the link function $l_1$ and $l_2$. Please see section~\ref{Sec::proof} for more details.

In Theorem~\ref{mainThConvergence}, we expressed the upper bounds on the estimation error in terms of the constituent vector, $t$. It is easy to translate these results in terms of the component vectors $w$ and $z$ using the triangle inequality:
\begin{align*}
\max\{\|w^0-w^*\|_2,\|z^0-z^*\|_2\}\leq\|t^0-t^*\|_2\leq\|w^0-w^*\|_2 + \|z^0-z^*\|_2.
\end{align*}
See Section~\ref{Sec::proof} for proofs and futher details.

\section{Experimental Results}
\label{Sec::Res}

In this section, we provide a range of numerical experiments for our proposed algorithms based on synthetic and real data. We compare the performance of \textsc{OneShot} and \textsc{DHT} with a LASSO-type technique for demixing, as well as a heuristic version of \textsc{OneShot} based on soft thresholding (inspired by the approach proposed in~\cite{yang}). We call these methods \emph{Nonlinear convex demixing with LASSO} or \textsc{(NlcdLASSO)}, and \emph{Demixing with Soft Thresholding} or \textsc{DST}, respectively. Before describing our simulation results, we briefly describe these two methods. 

\textsc{NlcdLASSO} is a \new{heuristic} method \new{motivated by}~\cite{plan2014high}, although it was not explicitly developed in the demixing context. Using our notation from Section \ref{Sec::Perm} and \ref{Sec::Alg}, \textsc{NlcdLASSO} solves the following convex problem:
\begin{equation} \label{prob 6.1}
\begin{aligned}
& \underset{z,w}{\text{min}}
& &  \big{\|}\widehat{x}_{\lin} - [\Phi~\Psi][w ; z] \big{\|}_2\\
& \text{subject to} 
& & \|w\|_1\leq \sqrt{s},\quad\|z\|_1\leq \sqrt{s}.
\end{aligned}
\end{equation}
Here, $\widehat{x}_{\lin}$ denotes the proxy of $x$ (equal to $\frac{1}{m} A^T y$) and $s$ denotes the sparsity level of signals $w$ and $z$ in basis $\Phi$ and $\Psi$, respectively. The constraints in problem \eqref{prob 6.1} are convex penalties reflecting the knowledge that $w$ and $z$ are $s$-sparse and have unit $\ell_2$-norm (since the
nonlinearity is unknown, we have a scale ambiguity, and therefore w.l.o.g. we can assume that the
underlying signals lie in the unit ball). 
The outputs of this algorithm are the estimates $\widehat{w}$, $\widehat{x}$, and $\widehat{x} = \Phi\widehat{w} + \Psi\widehat{z}.$

To solve the optimization problem in \eqref{prob 6.1}, we have used the SPGL1 solver \cite{Berg2008,spgl12007}. This solver can handle large scale problems, which is the scenario that we have used in our experimental evaluations. We impose the joint constraint $\|t\|_1 = \| [w; z] \|_1 \leq 2\sqrt{s}$ which is a slight relaxation of the constraints in ~\ref{prob 6.1}. The upper-bound of $\sqrt{s}$ in the constraints is a worst-case criterion; therefore, for a fairer comparison, we also include simulation results with the constraint $\|t\|_1\leq \varrho$, where $\varrho$ has been tuned to the best of our ability. 

On the other hand, \textsc{DST} solves the optimization problem~\eqref{optprob} via a convex relaxation of the sparsity constraint. In other words, this method attempts to  solve the following relaxed version of the problem~\eqref{optprob}:

\begin{equation}
\begin{aligned}
& \underset{t}{\text{min}}
\ \ \frac{1}{m}\sum_{i=1}^m \Theta(a_i^T\Gamma t) - y_i a_i^T\Gamma t + \beta'\|t\|_1,
\end{aligned}
\end{equation}
where $\|t\|_1$ represents $l_1$-norm of the constituent vector $t$ and $\beta' >0$ denotes the tuning parameter. The solution of this problem at iteration $k$ is given by soft thresholding operator as follows:
\begin{align*}
t^{k+1} = S_{\beta'\eta^{\prime}}(t^k - \eta^{\prime}\nabla F(t^k)),
\end{align*}
where $\eta^{\prime}$ denotes the step size, and the soft thresholding operator, $S_{\lambda}(.)$ is given by:
\[
    S_{\lambda}(y) =
\begin{cases}
    y-\lambda \, , \ & \text{if } y> \lambda\\
    0        \, ,       \ & \text{if } |y| \leq\lambda\\
    y+\lambda  \, , \ & \text{if } y< -\lambda.
\end{cases}
\]

Both \textsc{OneShot} and \textsc{NlcdLASSO} do not assume knowledge of the link function, and consequently return a solution up to a scalar ambiguity. Therefore, to compare performance across algorithms, we use the (scale-invariant) \textit{cosine similarity} between the original superposition signal $x$ and the output of a given algorithm $\widehat{x}$ defined as: 
$$\cos(x,\widehat{x}) = \frac{x^T\widehat{x}}{\|x\|_2\|\widehat{x}\|_2}.$$ 
 
\subsection{Synthetic Data}
As discussed above, for successful recovery we require the constituent signals to be sufficiently incoherent. To achieve this, we choose $\Phi$ to be the 1D Haar wavelets basis, and $\Psi$ to be the noiselet basis\footnote{These bases are known to be maximally incoherent relative to each other~\cite{coifman2001noiselets}}. For the measurement operator $A$, we choose a partial DFT matrix. Such matrices are known to have similar recovery performance as random Gaussian matrices, but enable fast numerical operations~\cite{candes2006robust}. Also, we present our experiments based on both non-smooth as well as differentiable link functions. For the non-smooth case, we choose $g(x) = \text{sign}(x)$; here, we only present recovery results using \textsc{OneShot} and \textsc{NlcdLASSO} since in our analysis \textsc{DHT} and \textsc{DST} can only handle smooth link functions. 

The results of our first experiment are shown in Figure~\ref{figCos}(a) and Figure~\ref{figCos}(b). The test signal is generated as follows: set length $n={2^{20}}$, and generate the vectors $w$ and $z$  by randomly selecting a signal support with $s$ nonzero elements, and populating the nonzero entries with random $\pm 1$ coefficients. The plot illustrates the performance of \textsc{Oneshot} and \textsc{NlcdLASSO} measured by the cosine similarity for different choices of sparsity level $s$, where the nonlinear link function is set to $g(x) = \text{sign(x)}$ \new{and we have used both $\|t\|_1\leq 2\sqrt{s}$ and $\|t\|_1\leq \varrho$ constraints}. The horizontal axis denotes an increasing number of measurements. Each data point in the plot is obtained by conducting a Monte Carlo experiment in which a new random measurement matrix $A$ is generated, recording the cosine similarity between the true signal $x$ and the reconstructed estimate and averaging over $20$ trials. 

\begin{figure}[t]
\begin{center}
\begin{tabular}{ccc}
\includegraphics[trim = 5mm 2mm 5mm 5mm, clip, width=0.32\linewidth]{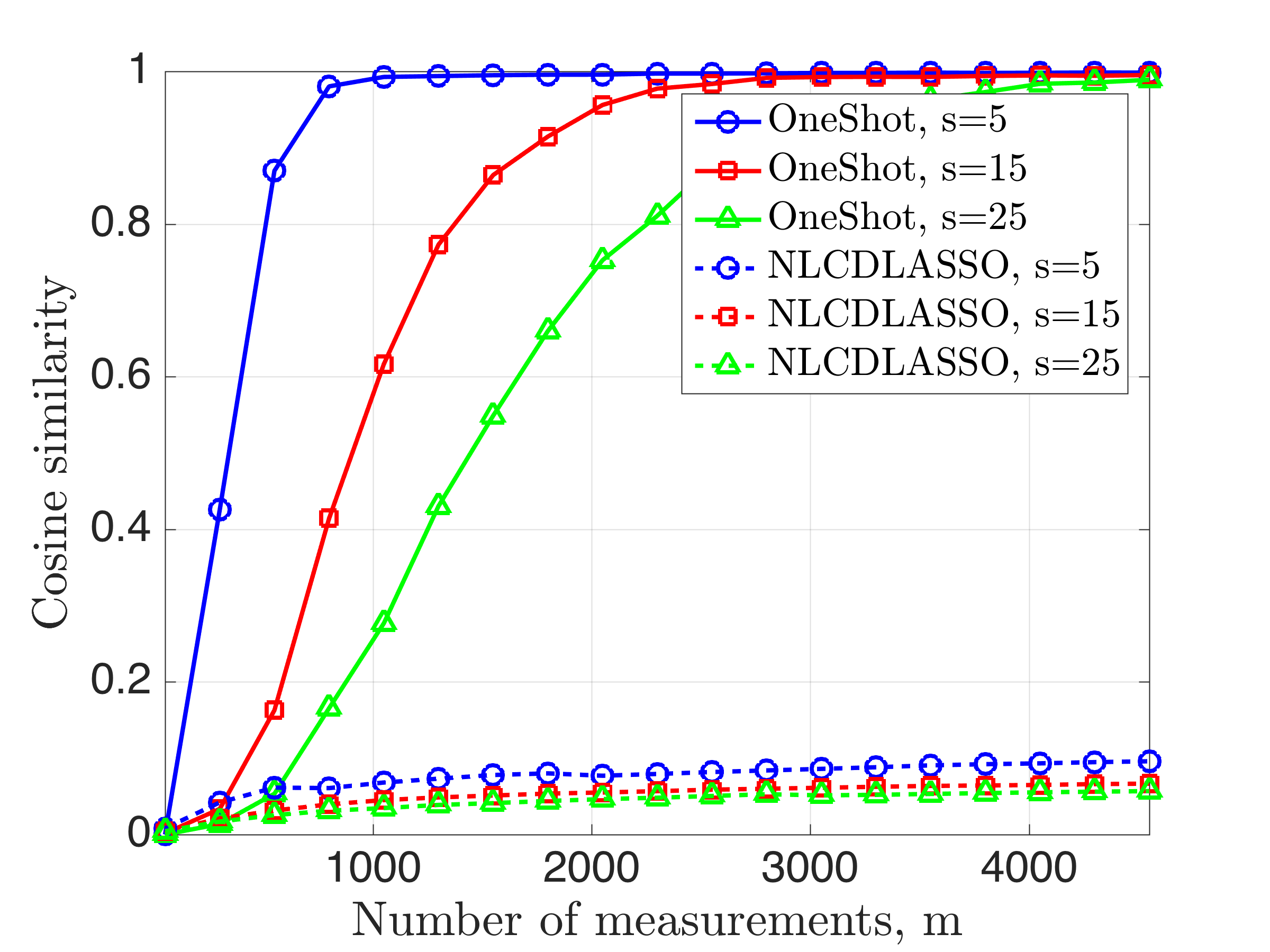} & 
\includegraphics[trim = 5mm 2mm 5mm 5mm, clip, width=0.32\linewidth]{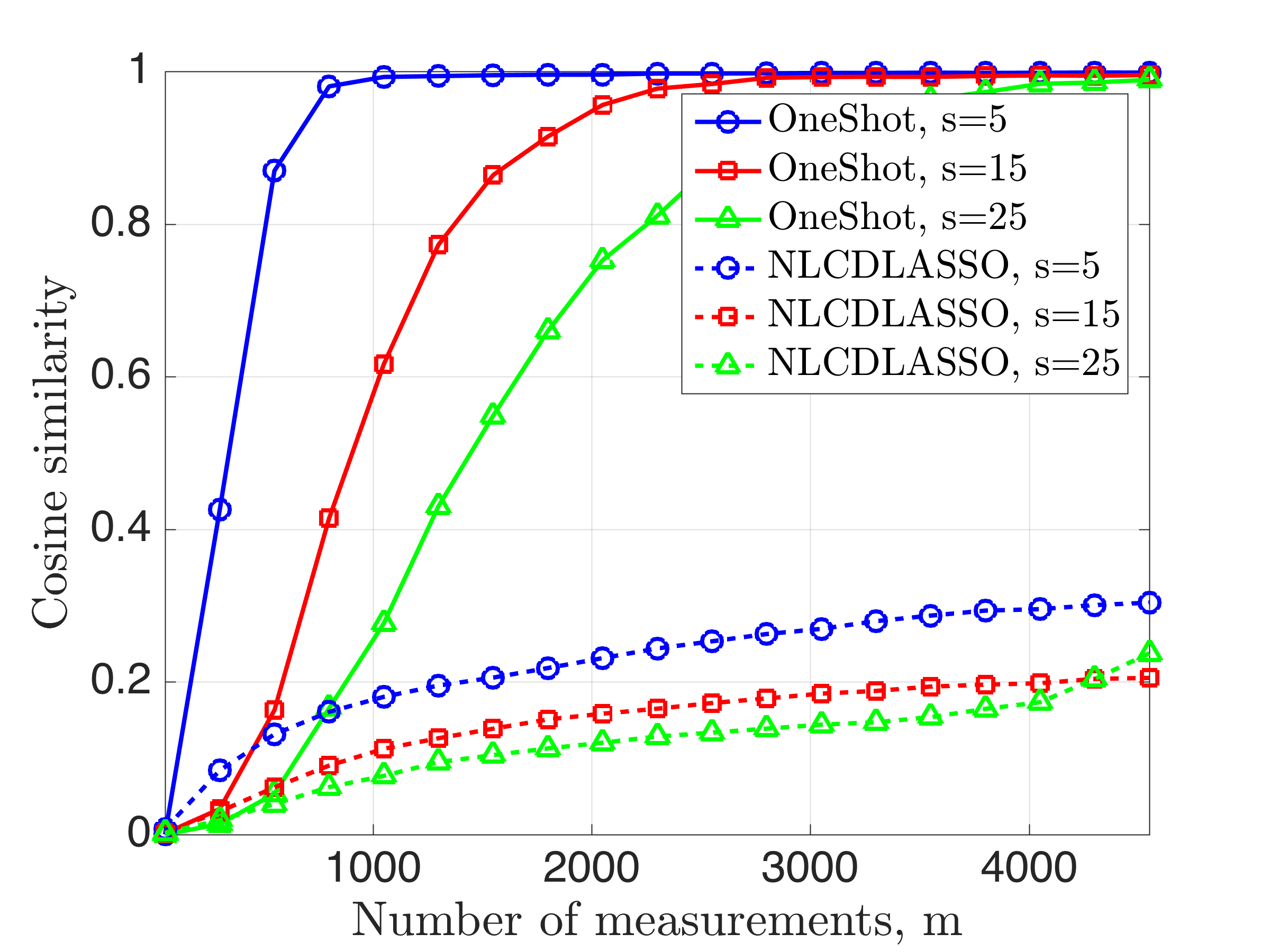} & 
\includegraphics[trim = 5mm 1mm 5mm 5mm, clip, width=0.32\linewidth]{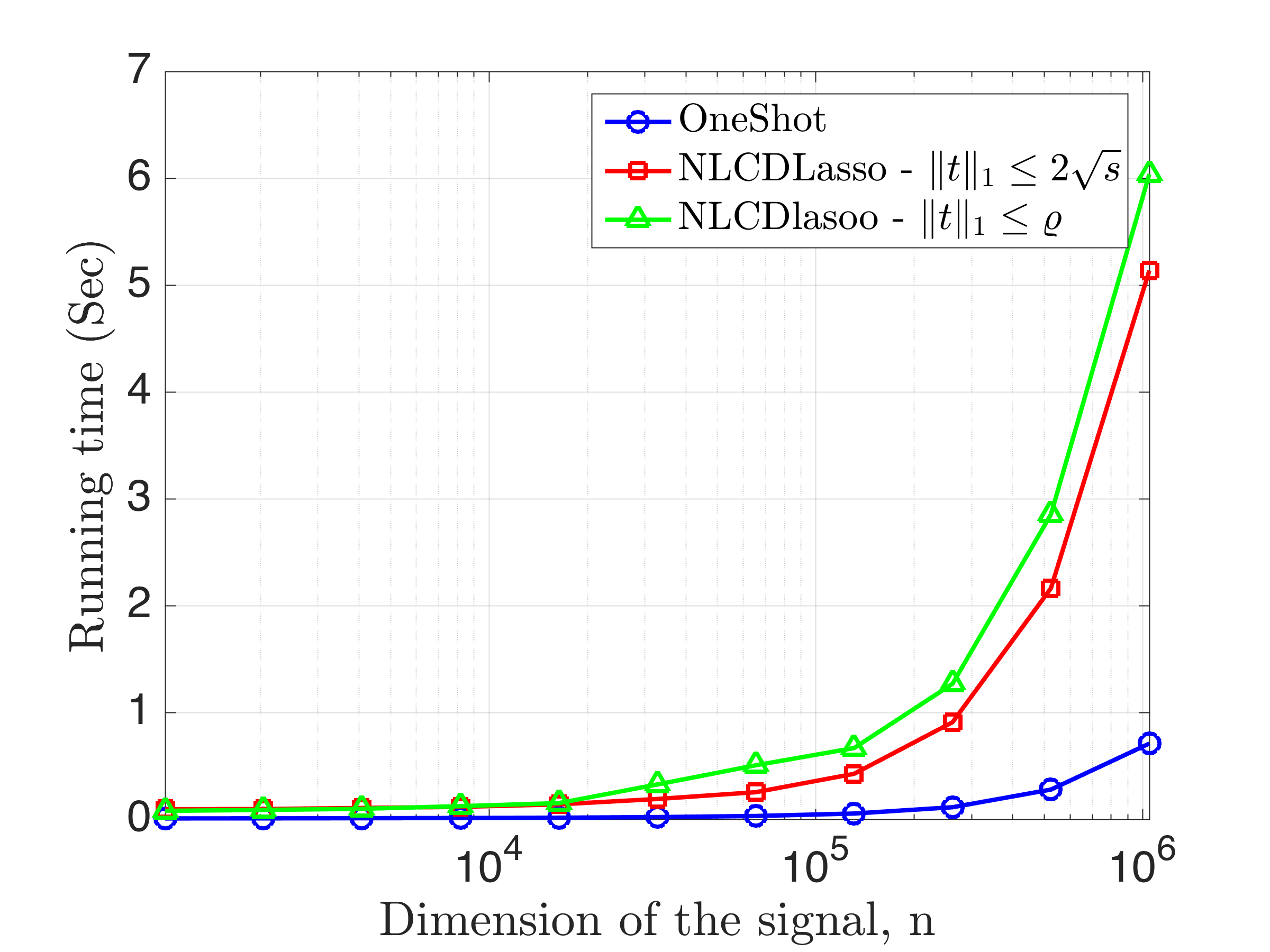} \\
(a)  & (b)  & (c) 
\end{tabular}
\end{center}
\vskip -.2in\caption{\small\emph{\new{Performance of \textsc{OneShot} and \textsc{NlcdLASSO} according to the \textsc{Cosine Similarity} for different choices of sparsity level $s$ for $g(x) = sign(x)$. (a) \textsc{NlcdLASSO} with $\|t\|_1\leq 2\sqrt{s}$. (b) \textsc{NlcdLASSO} with $\|t\|_1\leq \varrho$. (c) Comparison of running times of \textsc{OneShot} with \textsc{NlcdLASSO}.} }}
\label{figCos}
\end{figure}

As we can see, notably, the performance of \textsc{NlcdLASSO} is worse than \textsc{OneShot} for any fixed choice of $m$ and $s$ no matter what upper bound we use on $t$. Even when the number of measurements is high (for example, at $m = 4550$ in plot (b)), we see that \textsc{OneShot} outperforms \textsc{NlcdLASSO} by a significant degree. In this case, \textsc{NlcdLASSO} is at least $70\%$ worse in terms of signal estimation quality, while \textsc{OneShot} recovers the (normalized) signal perfectly. This result indicates the inefficiency of \textsc{NlcdLASSO} for nonlinear demixing.

%

Next, we contrast the running time of both algorithms, illustrated in Figure \ref{figCos}(c). In this experiment, we measure the wall-clock running time of the two recovery algorithms (\textsc{OneShot} and \textsc{NlcdLASSO}), by varying signal size $x$ from $n = 2^{10}$ to $n = 2^{20}$.  Here, we set $m = 500$, $s=5$, and the number of Monte Carlo trials to $20$. Also, the nonlinear link function is considered as $g(x)=\text{sign}(x)$. As we can see from the plot, \textsc{OneShot} is at least $6$ times faster than \textsc{NlcdLASSO} when the size of signal equals to $2^{20}$. Overall, \textsc{OneShot} is efficient even for large-scale nonlinear demixing problems.  We mention that in the above setup, the main computational costs incurred in \textsc{OneShot} involve a matrix-vector multiplication followed by a thresholding step, both of which can be performed in time that is \emph{nearly-linear} in terms of the signal length $n$ for certain choices of $A,\Phi,\Psi$
%

Next, we turn to differentiable link functions. In this case, we generate the constituent signal coefficient vectors, $w, z$ with $n= 2^{16}$, and compare performance of the four above algorithms. The nonlinear link function is chosen to be $g(x) = 2x + \sin(x)$; it is easy to check that the derivative of this function is strictly bounded between $l_1 = 1$ and $l_2 = 3$. The maximal number of iterations for both \textsc{DHT} and \textsc{DST} is set to to $1000$ with an early stopping criterion if convergence is detected. The step size is hard to estimate in practice, and therefore is chosen by manual tuning such that both \textsc{DHT} and \textsc{DST} obtain the best respective performance.

 Figure~\ref{PT} illustrates the performance of the four algorithms in terms of \emph{phase transition} plots, following~\cite{mccoyTropp2014}. In these plots, we varied both the sparsity level $s$ and the number of measurements $m$. For each pair $(s,m)$, as above we randomly generate the test superposition signal by choosing both the support and coefficients of $x$ at random, as well as the measurement matrix. We repeat this experiment over $20$ Monte Carlo trials.  We calculate the empirical probability of successful recovery as the number of trials in which the output cosine similarity is greater than $0.99$. Pixel intensities in each figure are normalized to lie between 0 and 1, indicating the probability of successful recovery.


\begin{figure}[t]
\begin{center}
\begin{tabular}{cccc}
\hskip -.2in
\includegraphics[trim = 22mm 1mm 22mm 5mm, clip, width=0.24\linewidth]{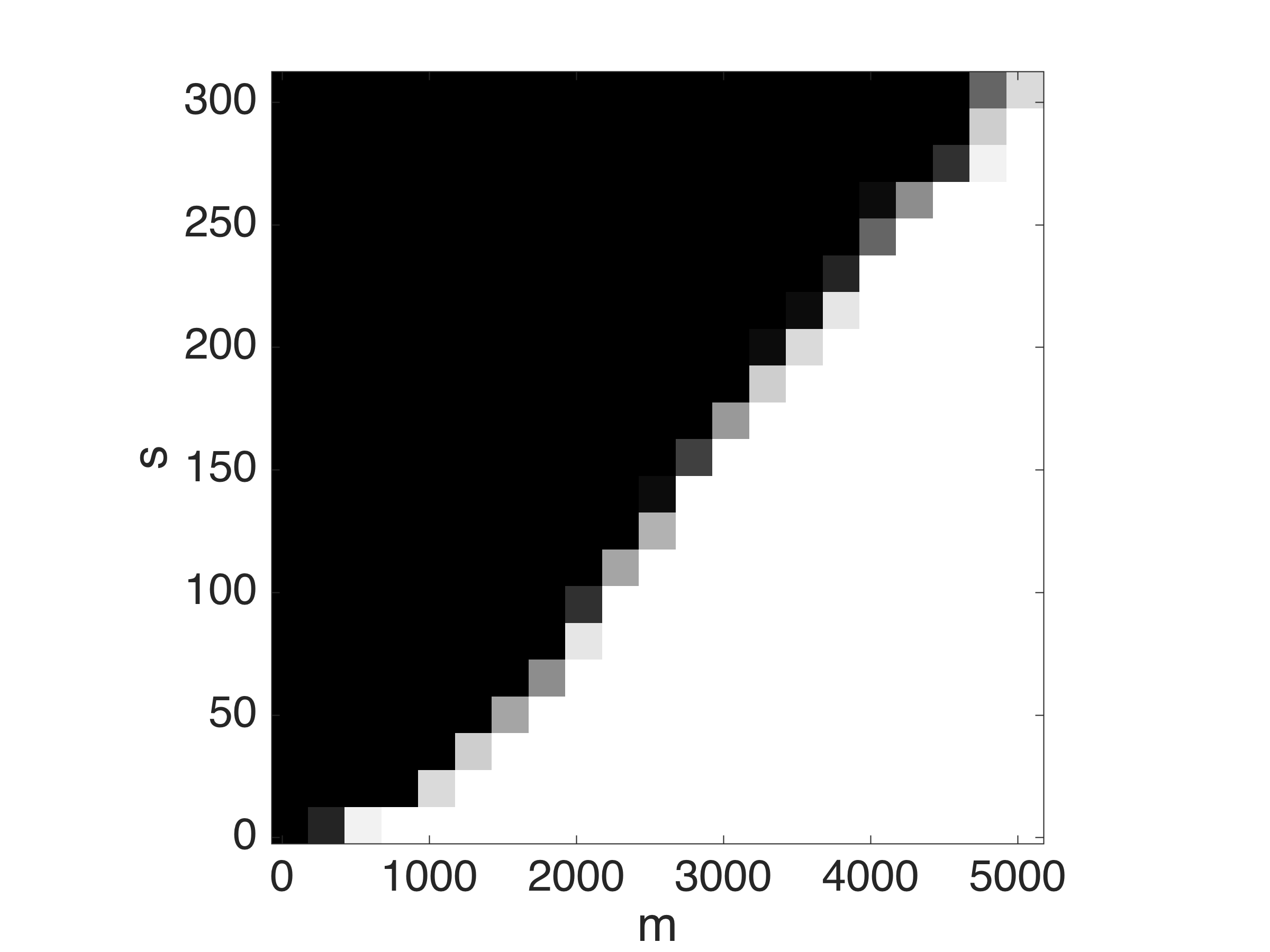} &
\includegraphics[trim = 22mm 1mm 22mm 5mm, clip, width=0.24\linewidth]{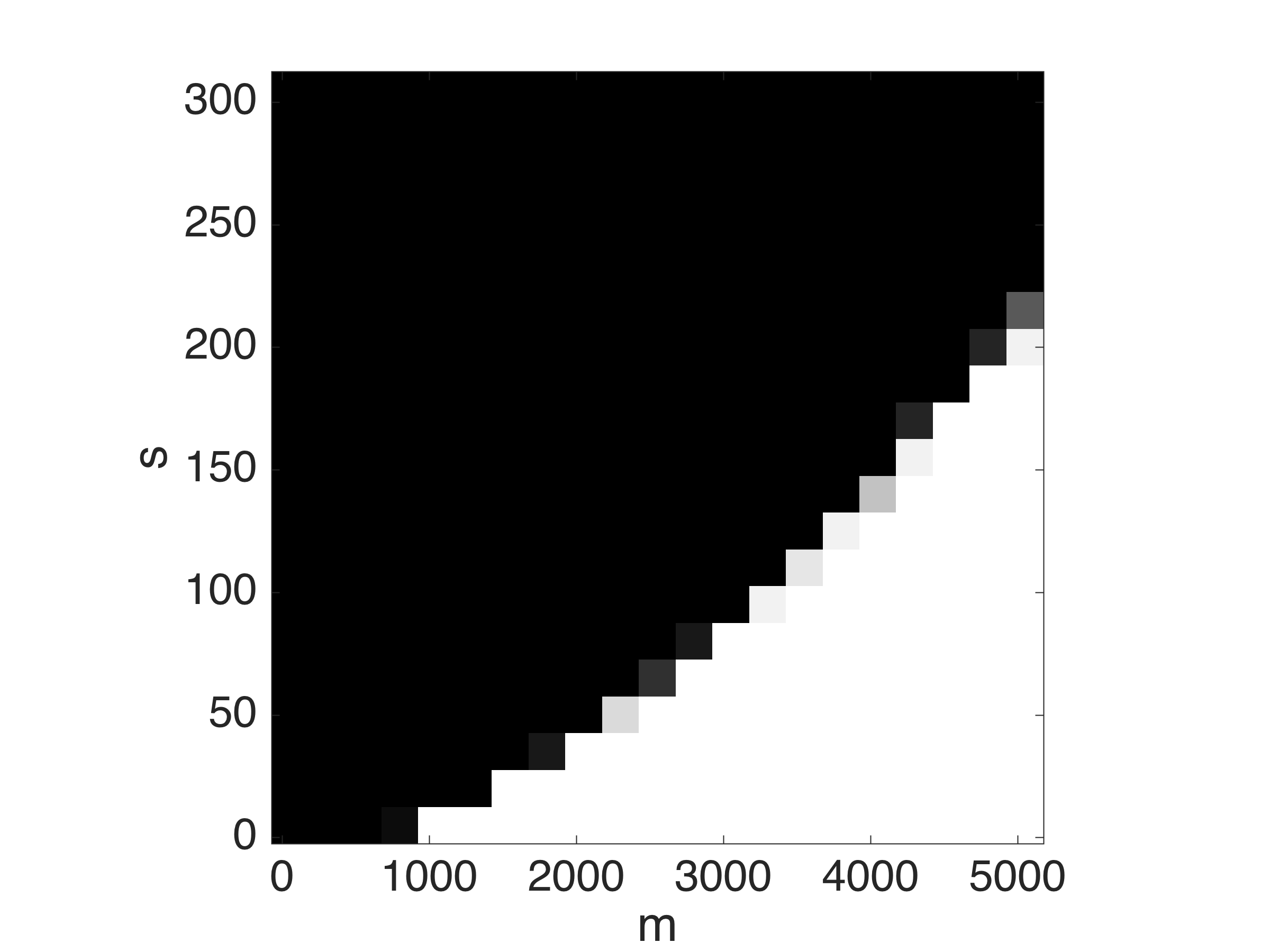} &
\includegraphics[trim = 22mm 1mm 22mm 5mm, clip, width=0.24\linewidth]{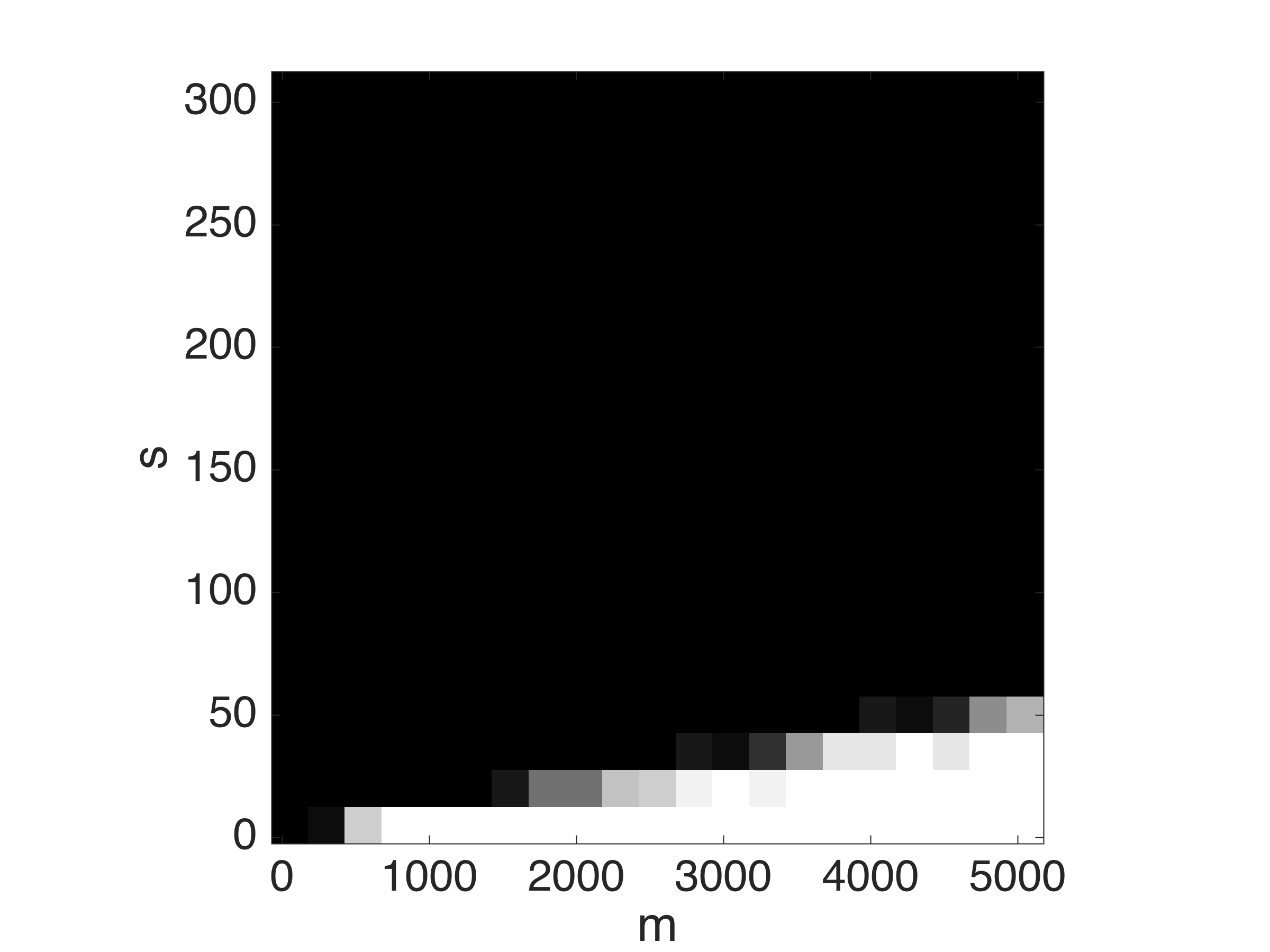} & 
\includegraphics[trim = 22mm 1mm 22mm 5mm, clip, width=0.24\linewidth]{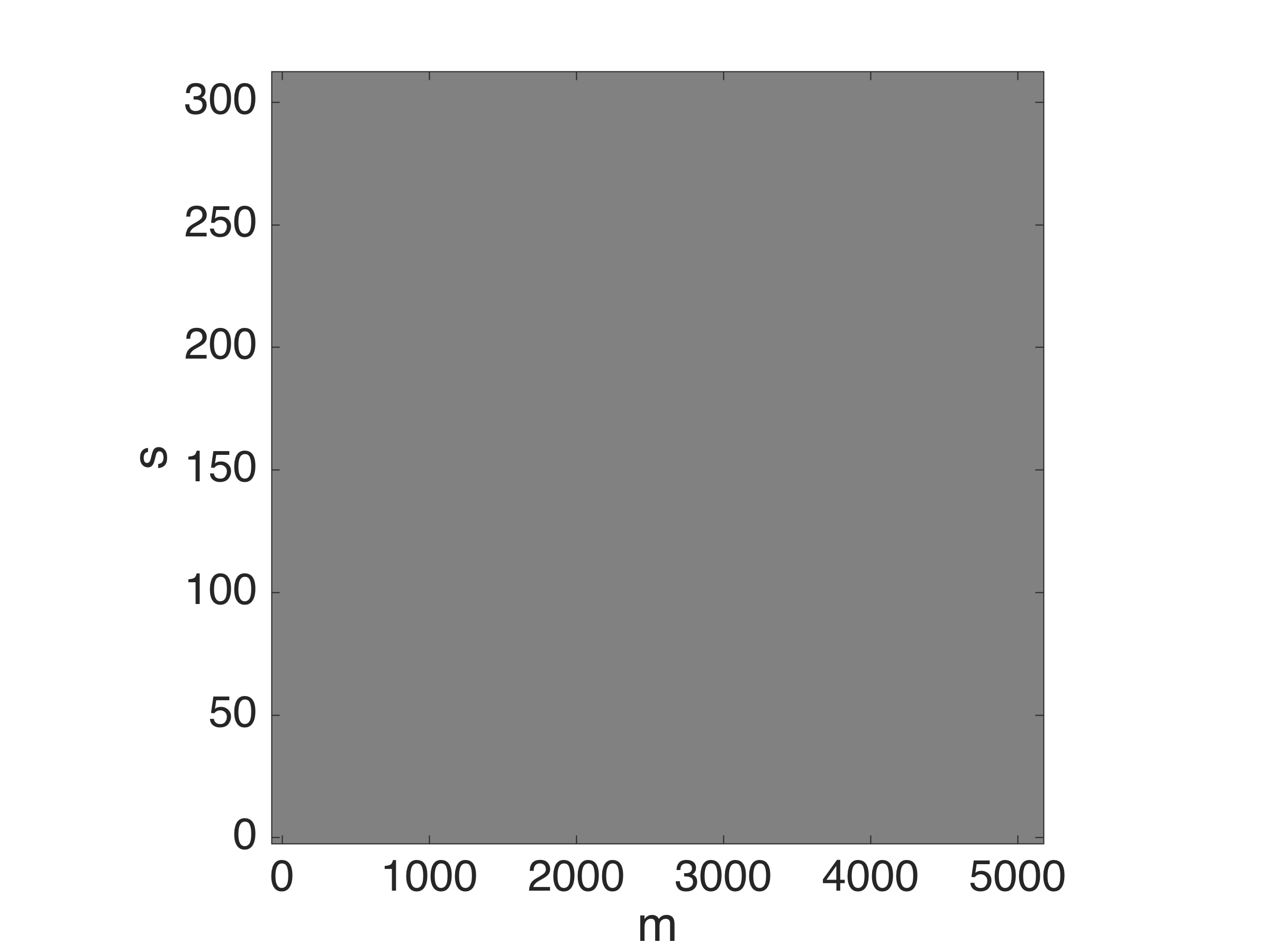} \\
 (a) DHT & (b) DST &
 (c) \textsc{OneShot} & (d) \textsc{NlcdLASSO} 
\end{tabular}
\end{center}
\vskip -.2in \caption{\emph{Phase transition plots of various algorithms for solving the demixing problem~\eqref{MainModell} as a function of sparsity level $s$ and number of measurements $m$ with cosine similarity as the criterion. Dimension of the signals $n= 2^{16}$.}}\label{PT}
\end{figure}

As we observe in Fig.~\ref{PT}, \textsc{DHT} has the best performance among the different methods, and in particular, outperforms both the convex-relaxation based methods. The closest algorithm to \textsc{DHT} in terms of the signal recovery is \textsc{DST}, while the \textsc{LASSO}-based method fails to recover the superposition signal $x$ (and consequently the constituent signals $w$ and $z$).  The improvements over \textsc{OneShot} are to be expected since as discussed before, this algorithm does not leverage the knowledge of the link function $g$ and is not iterative. 

In Fig.~\ref{4algProb}, we fix the sparsity level $s=50$ and plot the probability of recovery of different algorithms with a varying number of measurements. The number of Monte Carlo trials is set to 20 and the empirical probability of successful recovery is defined as the number of trials in which the output cosine similarity is greater than $0.95$. The nonlinear link function is set to be $g(x) = 2x + \sin(x)$ for figure (a) and $g(x) = \frac{1}{1+e^{-x}}$ for figure (b). As we can see, \textsc{DHT} has the best performance, while \textsc{NlcdLASSO} for figure (a) and \textsc{Oneshot}, and \textsc{NlcdLASSO} for figure (b) cannot recover the superposition signal even with the maximum number of measurements.

\begin{figure}[h]
\begin{center}
\begin{tabular}{cc}
\includegraphics[trim = 5mm 65mm 22mm 60mm, clip, width=0.42\linewidth]{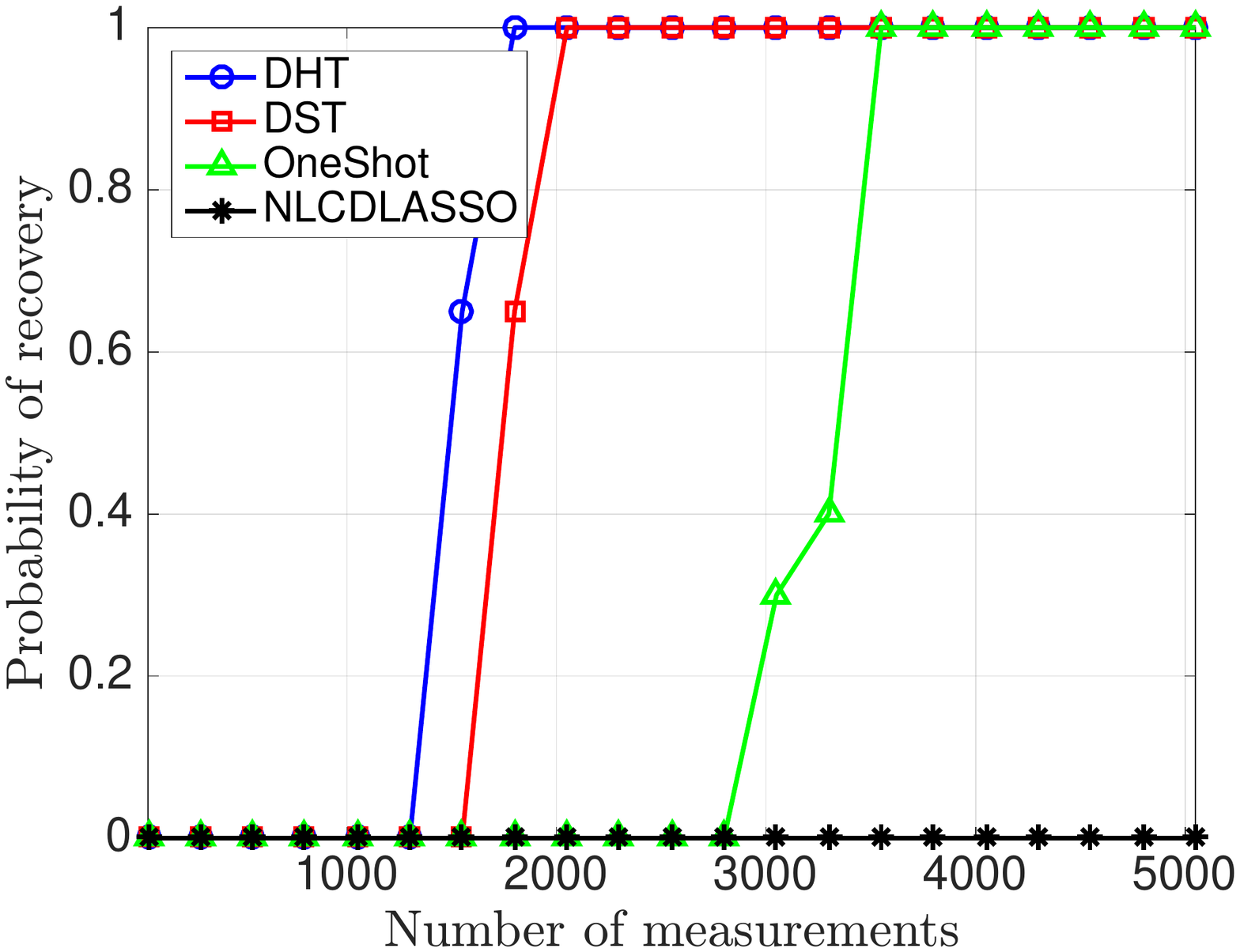}&
\includegraphics[trim = 5mm 65mm 22mm 60mm, clip, width=0.42\linewidth]{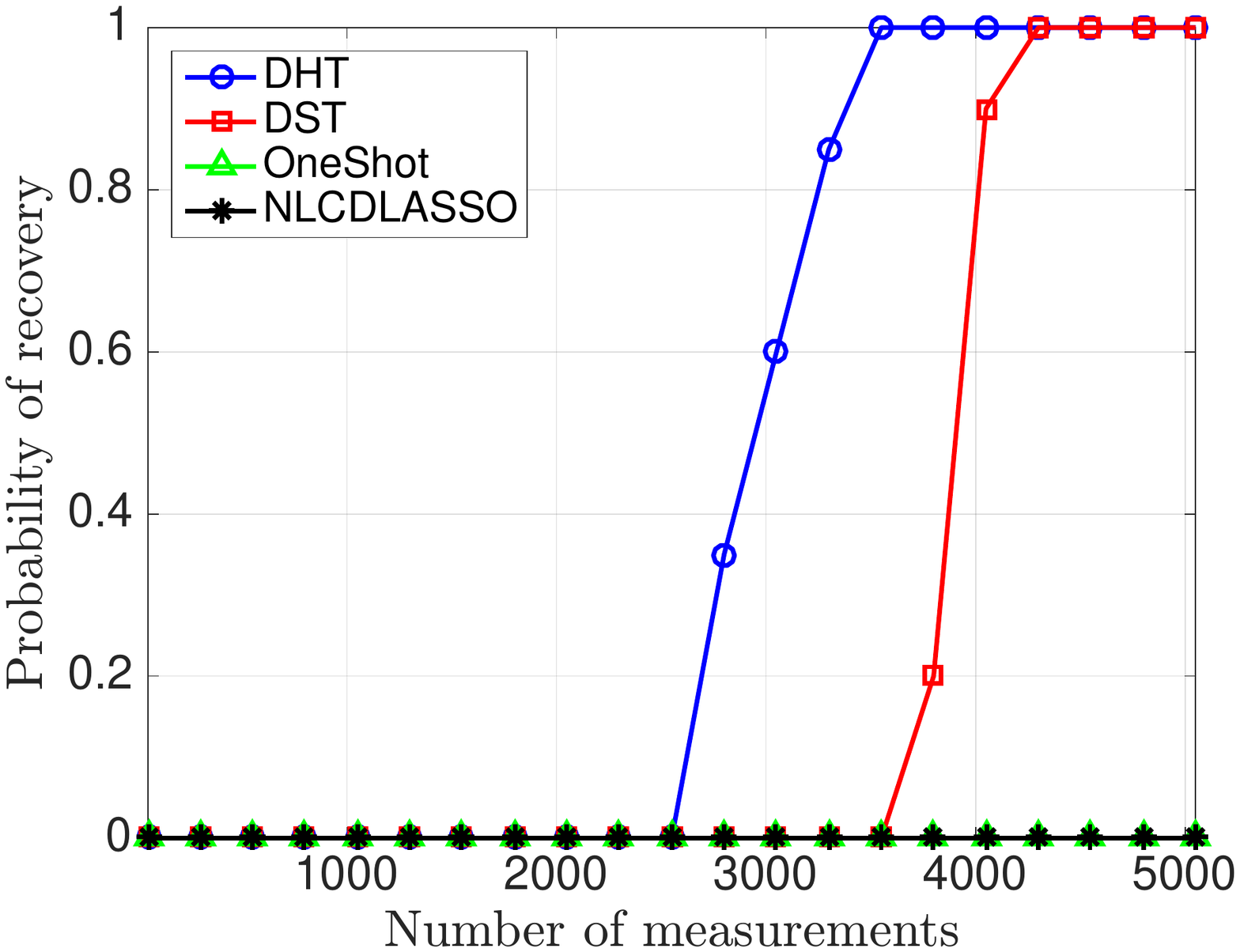}\\
(a) & (b) 
\end{tabular}
\end{center}
\vskip -.2in\caption{\emph{Probability of recovery for four algorithms; \textsc{DHT}, \textsc{STM}, \textsc{Oneshot}, and \textsc{NlcdLASSO}. Sparsity level is set to $s =50$ and dimension of the signals equal to $n= 2^{16}$. (a) $g(x) = 2x + \sin(x)$, (b) $g(x) = \frac{1}{1+e^{-x}}$.}}
\label{4algProb}
\end{figure}

\subsection{Real Data}

In this section, we provide representative results on real-world 2D image data using \textsc{Oneshot} and \textsc{NlcdLASSO} for non-smooth link function given by $g(x) = \text{sign}(x)$. In addition, we illustrate results for all four algorithms using smooth $g(x) = \frac{1-e^{-x}}{1+e^{-x}}$ as our link function.

We begin with a $256\times 256$ test image. First, we obtain its 2D Haar wavelet decomposition and retain the $s=500$ largest coefficients, denoted by the $s$-sparse vector $w$. Then, we reconstruct the image based on these largest coefficients, denoted by $\widehat{x} = \Phi w$. Similar to the synthetic case, we generate a noise component in our superposition model based on 500 noiselet coefficients $z$. In addition, we consider a parameter which controls the strength of the noiselet component contributing to the superposition model. We set this parameter to 0.1. Therefore, our test image $x$ is given by $x = \Phi w + 0.1\Psi z$.

\vskip -.55in
\begin{figure}[!h]
\begin{center}
\begin{tabular}{ccc}
\includegraphics[trim = 27mm 83mm 35mm 22mm, clip, width=0.282\linewidth]{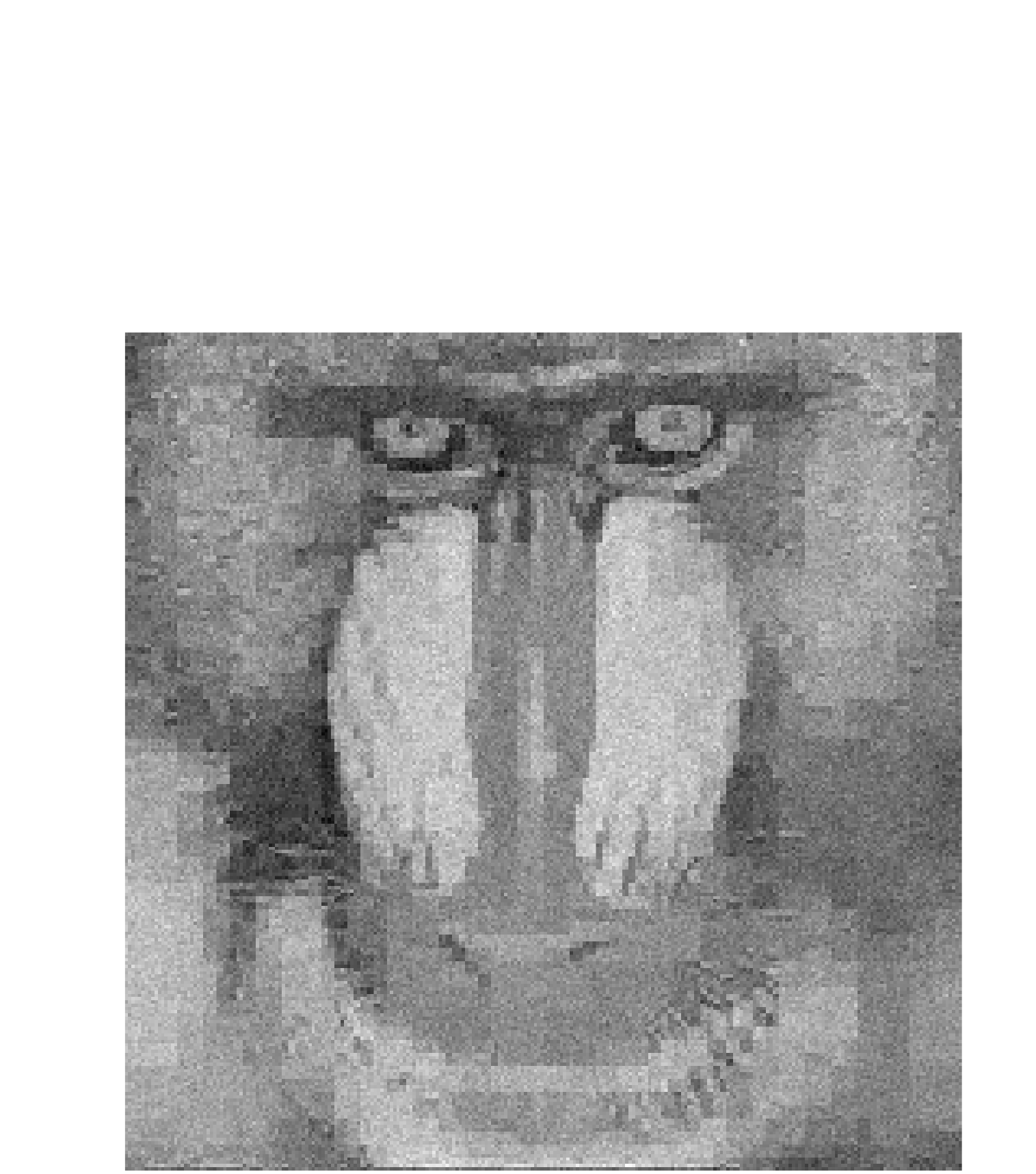} &
\includegraphics[trim = 27mm 83mm 35mm 22mm, clip, width=0.282\linewidth]{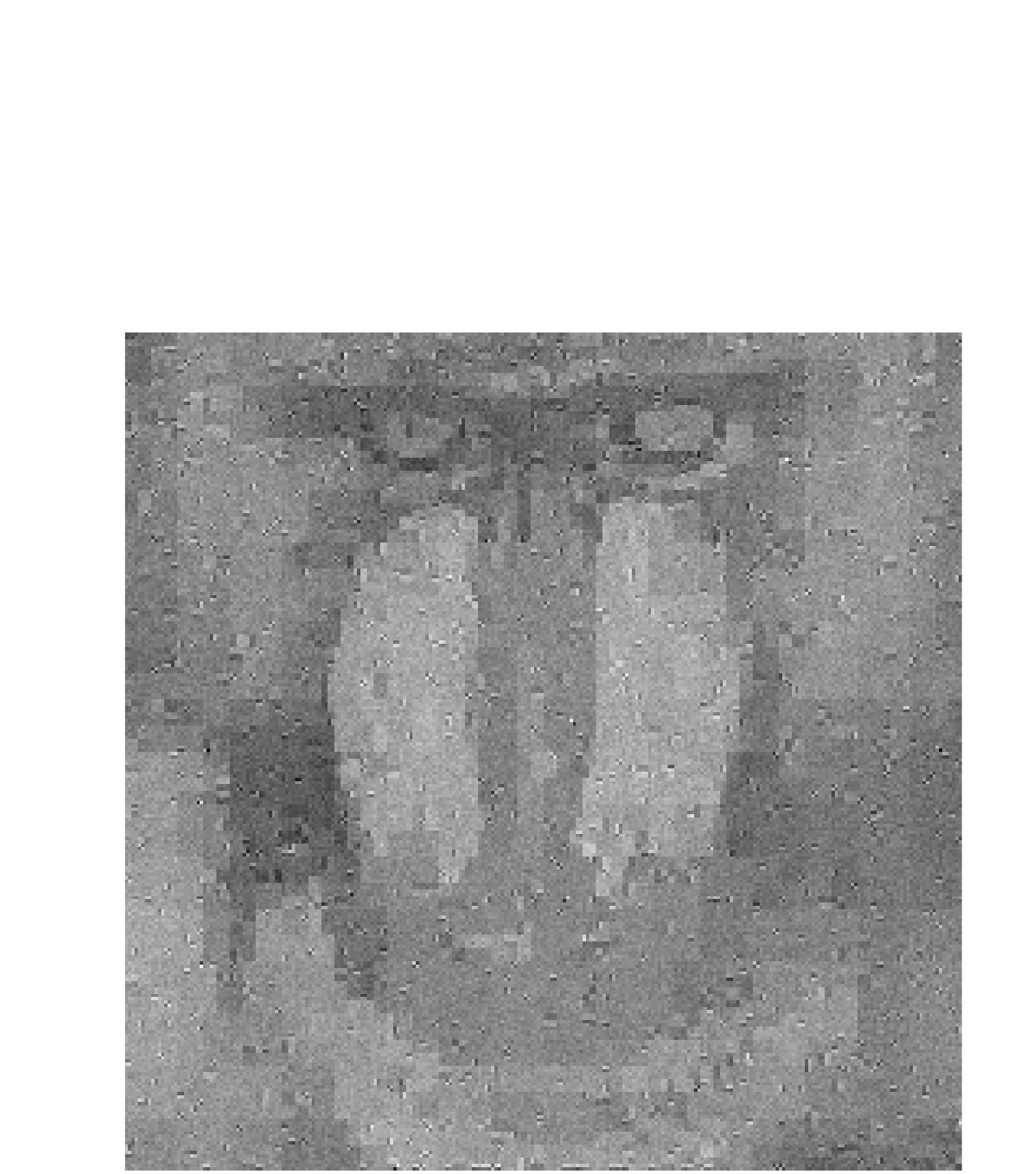} &
\hskip .21in\includegraphics[trim = 2mm 0mm 1mm 1.5mm, clip, width=0.228\linewidth]{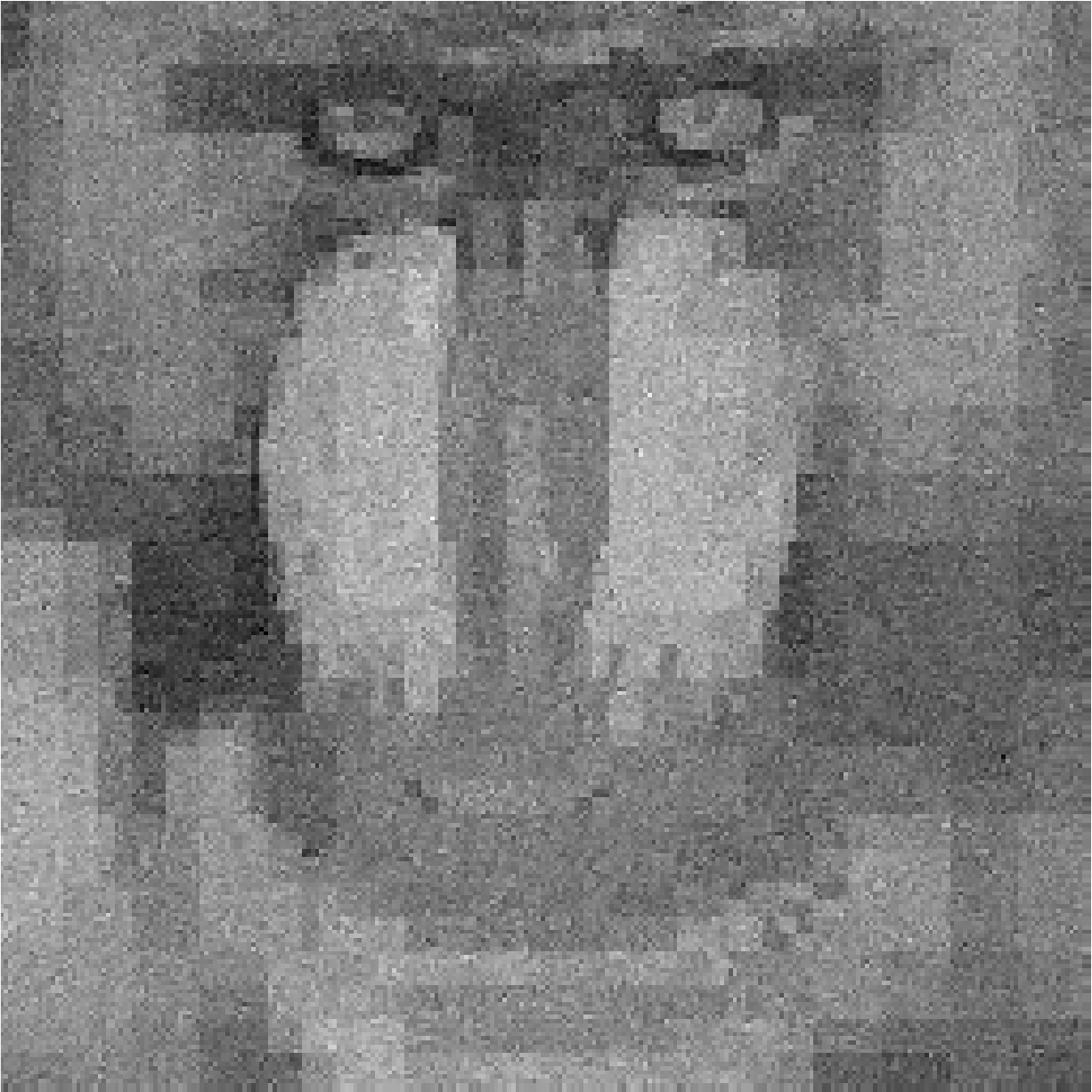} \\
$x$ & $\widehat{x}~(\textsc{OneShot})$ &  $\widehat{x}~(\textsc{NlcdLasso})$  
\end{tabular}
\end{center}
\vskip -.2in\caption{\emph{Comparison of \textsc{Oneshot} and \textsc{NlcdLASSO} for real 2D image data from nonlinear under-sampled observations. Parameters: $n = 256 \times 256, s = 500, m = 35000, g(x) = sign(x)$.}}\label{real2d}
\end{figure}

Figure~\ref{real2d} illustrates both the true and the reconstructed images $x$ and $\widehat{x}$ using \textsc{Oneshot} and \textsc{NlcdLASSO}. The number of measurements is set to $35000$ (using subsampled Fourier matrix with $m=35000$ rows). From visual inspection we see that the reconstructed image, $\widehat{x}$, using \textsc{Oneshot} is better than the reconstructed image by \textsc{NlcdLASSO}. Quantitatively, we also calculate Peak signal-to-noise-ratio (PSNR) of the reconstructed images using both algorithms relative to the test image, $x$. We obtain
PSNR of 19.8335 dB using \textsc{OneShot}, and a PSNR of 17.9092 dB using {\textsc{NlcdLASSO}},
again illustrating the better performance of \textsc{Oneshot} compared to \textsc{NlcdLASSO}.

Next, we show our results using a differentiable link function. For this experiment, we consider an astronomical image illustrated in Fig.~\ref{fig:StarGla}. This image includes two components; the ``stars" component, which can be considered to be sparse in the identity basis ($\Phi$), and the ``galaxy" component which are sparse when they are expressed in the discrete cosine transform basis ($\Psi$). The superposition image $x = \Phi w + \Psi z$ is observed using a subsampled Fourier matrix with $m=15000$ rows multiplied with a diagonal matrix with random $\pm 1$ entries~\cite{krahmer2011new}. Further, each measurement is nonlinearly transformed by applying the (shifted) logistic function $g(x) = \frac{1}{2}\frac{1-e^{-x}}{1+e^{-x}}$ as the link function. In the recovery procedure using DHT, we set the number of iterations to $1000$ and step size $\eta'$ to $150000$. As is visually evident, our proposed DHT method is able to reliably recover the component signals.

\begin{figure}
\begin{center}
\begingroup
\setlength{\tabcolsep}{.1pt} 
\renewcommand{\arraystretch}{.1} 
\begin{tabular}{ccc}      
\includegraphics[trim = 27mm 5mm 35mm 5mm, clip, width=0.27\linewidth]{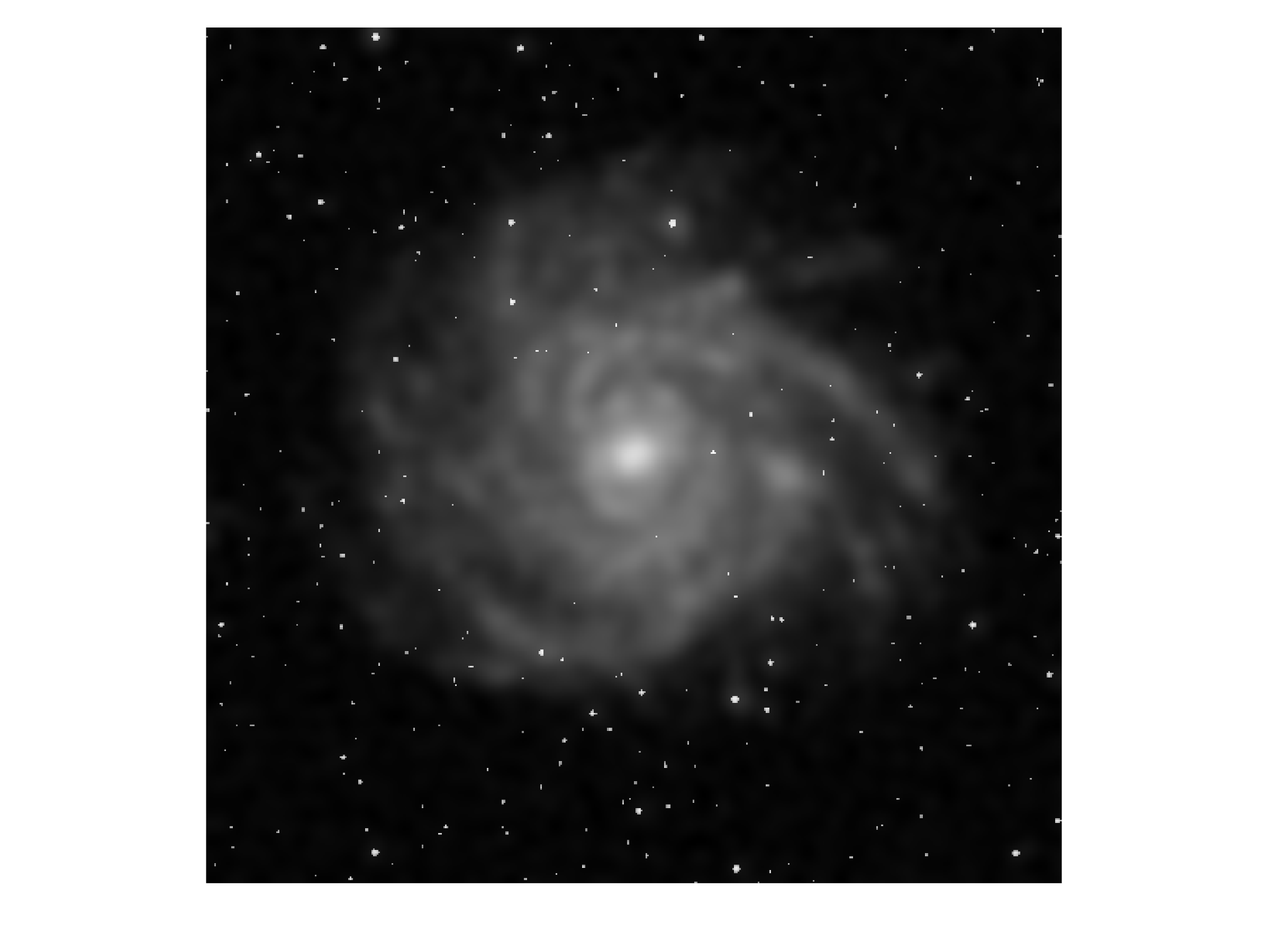}&
\includegraphics[trim = 27mm 5mm 35mm 5mm, clip, width=0.27\linewidth]{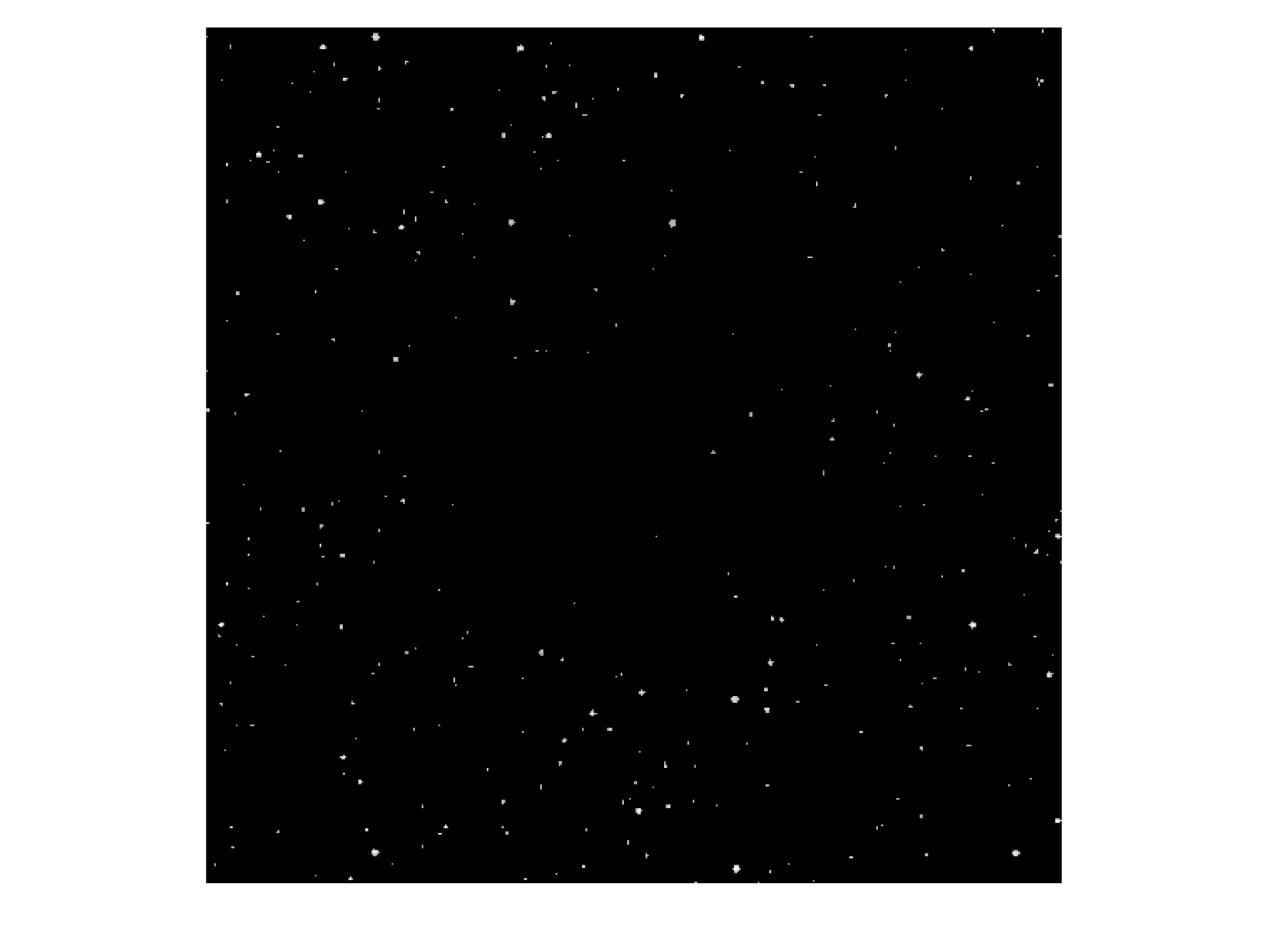}&
\includegraphics[trim = 27mm 5mm 35mm 5mm, clip, width=0.27\linewidth]{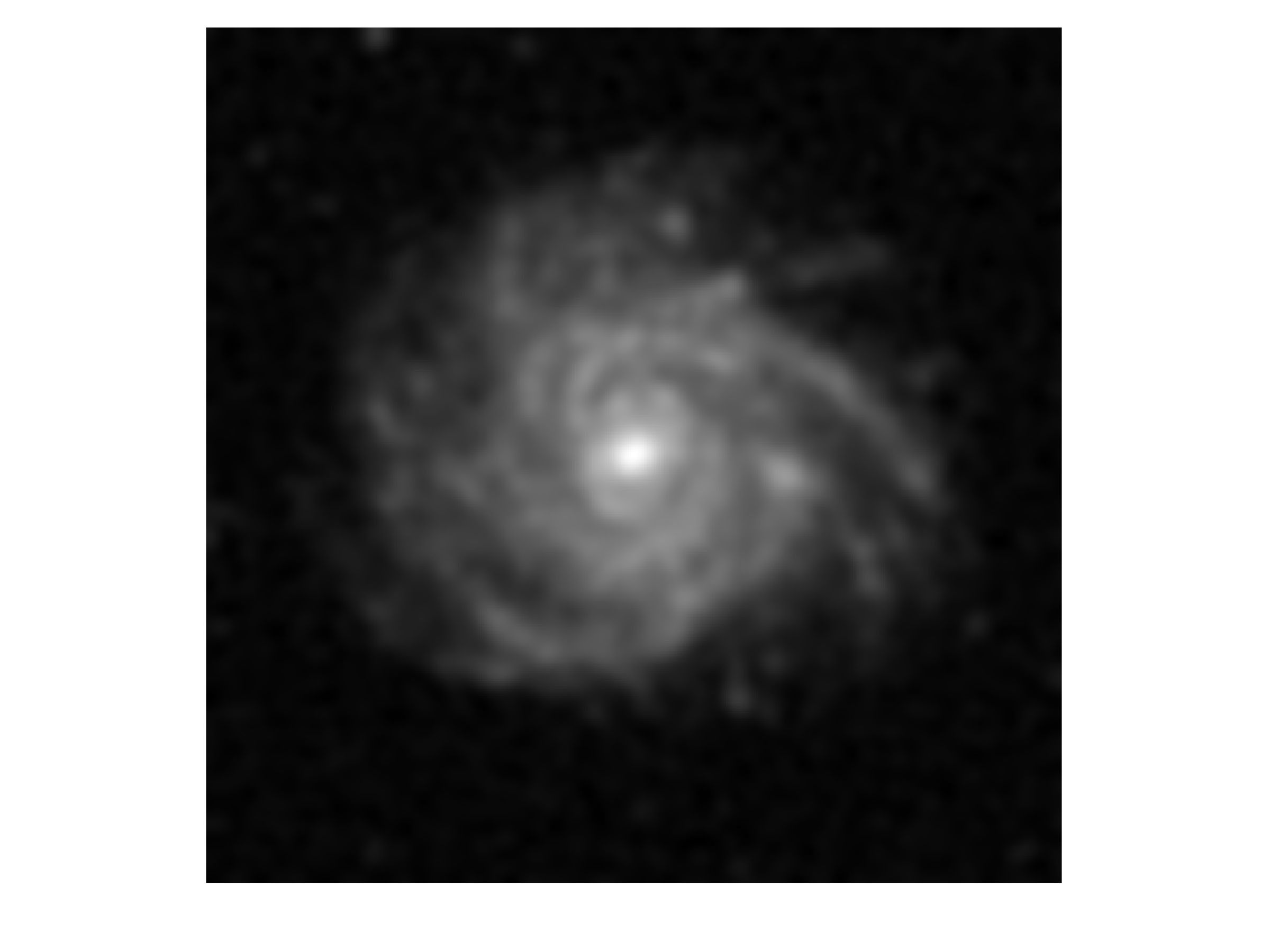}\\
(a) Original {\small$x$} & (b) {\small$\Phi(\widehat{w})$} & {\small$\Psi(\widehat{z})$}
\end{tabular}
\endgroup
\end{center}
\caption{\emph{Demixing a real 2-dimensional image from nonlinear observations with \textsc{DHT}. Parameters: $n = 512 \times 512, s = 1000, m = 15000, g(x) = \frac{1}{2}\frac{1-e^{-x}}{1+e^{-x}}$. Image credits: NASA, \cite{mccoy2014convexity}.}}
\label{fig:StarGla}
\end{figure}

\section{Proofs}
\label{Sec::proof}

In this section, we derive the proofs of our theoretical results stated in Section~\ref{Sec::Alg}. 


\subsection{Analysis of \textsc{OneShot}}
Our analysis mostly follows the techniques of \cite{plan2014high}. However, several additional complications in the proof arise due to the structure of the demixing problem. 
As a precursor, we need the following lemma from geometric functional analysis, restated from~\cite{plan2014high}.

\begin{lemma}\label{lemma 5.1}
Assume $K$ is a closed star-shaped set. Then for $u\in K$, and $a\in \mathbb{R}^n$, one has the following result $\forall~t > 0$:
\begin{equation}\label{eq 5.1}
\|\mathcal{P}_K(a) - u\|_2\leq\max \left(t, \frac{2}{t}\|a - u\|_{K_t^o}\right).
\end{equation}
\end{lemma}

We also use the following result of~\cite{plan2014high}.

\begin{claim}(Orthogonal decomposition of $a_i$.)\label{re 4.5}
Suppose we decompose the rows of $A$, $a_i$, as:
\begin{align}
 a_i =  \langle{a_i,\bar{x}}\rangle\bar{x} + b_i,
\end{align} where $b_i\in\mathbb{R}^n$ is orthogonal to $\bar{x}$. Then we have $b_i\sim\mathcal{N}(0,I_{x^{\perp}})$ since $a_i\sim\mathcal{N}(0,I)$. Also, $I_{x^{\perp}}=I-\bar{x}\bar{x}^T.$ Moreover, the measurements $y_i$ in equation \eqref{MainModell} and the orthogonal component $b_i$ are statistically independent.
\end{claim}

\begin{proof}[Proof of Theorem~\ref{thm:main}]
Observe that the magnitude of the signal $x$ may be lost due to the action of the nonlinear measurement function $f$ (such as the $\mathrm{sign}(\cdot)$ function). Therefore, our recovered signal $\widehat{x}$ approximates the true signal modulo a scaling factor. Indeed, for $\mu$ defined in Lemma~\ref{lemma 4.1}, we have:
\begin{align}\label{MainDecEq}
\|\widehat{x} - \mu\bar{x}\|_2 &= \|\Phi\widehat{w}+\Psi\widehat{z}-\alpha\mu\Phi{w}-\alpha\mu\Psi{z}\|_2 \nonumber\\
&\leq\|\Phi\|\|\widehat{w}-\mu\alpha{w}\|_2+\|\Psi\|\|\widehat{z}-\mu\alpha{z}\|_2 \nonumber\\
&\leq ({\rho}+\frac{2}{\rho}\|\Phi^*\widehat{x}_{\lin}-\mu\alpha{w}\|_{K_{\rho}^o}) + ({\rho}+\frac{2}{\rho}\|\Psi^*\widehat{x}_{\lin}-\mu\alpha{z}\|_{K_{\rho}^o}).
\end{align}
The equality comes from the definition of $\bar{x}$. The first inequality results from an application of the triangle inequality and the definition of the operator norm of a matrix, while the second inequality follows from Lemma \ref{lemma 5.1}. 
Now, it suffices to derive a bound on the first term in the above expression (since a similar bound will hold for the second term). This proves the first part of Theorem~\ref{thm:main}. We have:
\begin{align}
\|\Phi^*\widehat{x}_{\lin}-\mu\alpha{w}\|_{K_{\rho}^o} &= \|\Phi^*\frac{1}{m}\Sigma_i(y_i\langle a_i, \bar{x}\rangle\bar{x}+y_i b_i)-\mu\alpha{w}\|_{K_{\rho}^o}\nonumber\\
&\leq\|\Phi^*\frac{1}{m}\Sigma_i(y_i\langle a_i, \bar{x}\rangle\bar{x})-\mu\alpha{w}\|_{K_{\rho}^o}+\|\Phi^*\frac{1}{m}\Sigma_i y_i b_i\|_{K_{\rho}^o}\nonumber\\
&\leq\underbrace{\|\Phi^*\frac{1}{m}\Sigma_i(y_i\langle a_i, \bar{x}\rangle\bar{x})-\mu\Phi^*\bar{x}\|_{K_{\rho}^o}}_\text{$S_1$}
+ \underbrace{\|\mu\alpha\Phi^*\Psi{z}\|_{K_{\rho}^o}}_\text{$S_2$}\nonumber \\
&+\underbrace{\|\Phi^*\frac{1}{m}\Sigma_i y_i b_i\|_{K_{\rho}^o}}_\text{$S_3$} . \label{eq}
\end{align}
The first equality follows from Claim~\ref{re 4.5}, while the second and third inequalities result from the triangle inequality. Also:
\begin{align*}
S_1 &= \|\Phi^*\frac{1}{m}\Sigma_i(y_i\langle a_i, \bar{x}\rangle\bar{x})-\mu\Phi^*\bar{x}\|_{K_{\rho}^o} \\
&= \|(\frac{1}{m}\Sigma_i(y_i\langle a_i, \bar{x}\rangle-\mu) )\Phi^*\bar{x}\|_{K_{\rho}^o}\\
&= \lvert\frac{1}{m}\Sigma_i(y_i\langle a_i, \bar{x}\rangle-\mu) \rvert\|\Phi^*\bar{x}\|_{K_{\rho}^o} .\\
&\Longrightarrow\mathbb{E}(S_1^2) = \mathbb{E}(|\frac{1}{m}\Sigma_i(y_i\langle a_i, \bar{x}\rangle-\mu) |^2\|\Phi^*\bar{x}\|_{K_{\rho}^o}^2) .
\end{align*}
Define $\gamma_i\overset{\Delta}{=} y_i\langle a_i,\bar{x}\rangle - \mu_i$. Then,
 \begin{align*}
\mathbb{E}(|\frac{1}{m}\Sigma_i(y_i\langle a_i, \bar{x}\rangle-\mu) |^2)&=\mathbb{E}(\frac{1}{m^2}(\Sigma_i \gamma_i)^2) \\
& =\mathbb{E}(\frac{1}{m^2}(\sum_{i=}^{m} \gamma_i^2 + \Sigma_{i\neq j}\gamma_i\gamma_j)) \\
&=\frac{1}{m^2}(\sum_{i=1}^m \mathbb{E}\gamma_i^2) = \frac{1}{m}\mathbb{E}\gamma_1^2 = \frac{\sigma^2}{m},
\end{align*}
where $\sigma^2$ has been defined in Lemma \ref{lemma 4.1}. The third and last equalities follow from the fact that the $y_i$'s are independent and identically distributed. 
Now, we bound $\|\Phi^*\bar{x}\|_{K_{\rho}^o}^2$ as follows::
\begin{align}
\|\Phi^*\bar{x}\|_{K_{\rho}^o} &= \sup_{u\in (K-K)\cap {\rho}B_n^2}\langle\Phi^*\bar{x},u\rangle \nonumber \\
&= {\rho}\sup_{\substack{v_1\in \frac{1}{\rho}K, v_2\in \frac{1}{\rho}K\\\|v_i\|_2\leq 1, i=1,2}}\langle\Phi^*\bar{x},v_1-v_2\rangle \nonumber\\
& \leq 2{\rho}\sup_{\substack{\|a\|_0\leq s\\\|a\|_2\leq 1}}|\langle\Phi^*\bar{x},a\rangle| \nonumber \\
&\leq 2{\rho}(\sup_{\substack{\|a\|_0\leq s\\\|a\|_2\leq 1}}|\langle\alpha{w},a\rangle|+\sup_{\substack{\|a\|_0\leq s\\\|a\|_2\leq 1}}|\langle\alpha\Phi^*\Psi{z},a\rangle|) \nonumber\\
&\leq 2{\rho}(\alpha\|w\|_2+\sup_{\substack{\|a\|_0\leq s\\\|a\|_2\leq 1}}|\langle\alpha\Psi{z},\Phi a\rangle|)\leq2\alpha {\rho}(\|w\|_2+\|z\|_2\varepsilon).\nonumber 
\end{align}
The second inequality follows from \eqref{superpositoin} and the triangle inequality. The last inequality is resulted from an application of the Cauchy-Schwarz inequality and the definition of $\varepsilon$. As a result, we have:
\begin{align}
&\Longrightarrow\mathbb{E}(S_1^2)\leq 4\frac{\alpha^2 {\rho}^2\sigma^2}{m}\left(\|w\|_2+\|z\|_2\varepsilon\right)^2\label{eq 5.3}.
\end{align}
Similarly we can bound $S_2$ as follows:
\begin{align}
\mathbb{E}(S_2) = \mathbb{E}(\|\mu\alpha\Phi^*\Psi{z}\|_{K_{\rho}^o})&=\mathbb{E}(|\mu\alpha|\|\Phi^*\Psi{z}\|_{K_{\rho}^o}) \nonumber\\
&=|\mu\alpha|\|\Phi^*\Psi{z}\|_{K_{\rho}^o}=|\mu\alpha|\sup_{u\in (K-K)\cap {\rho}B_n^2}\langle\Psi{z},\Phi u\rangle \nonumber\\
&=|\mu\alpha|{\rho}\sup_{\substack{v_1\in \frac{1}{\rho}K, v_2\in \frac{1}{\rho}K\\\|v_i\|_2\leq 1, i=1,2}}\langle\Psi{z},\Phi(v_1-v_2)\rangle\leq 2\mu\alpha {\rho}\|z\|_2\varepsilon .\label{eq 5.4}
\end{align}
Finally, we give the bound for $S_3$. Define $\quad L\overset{\Delta}{=} \frac{1}{m}\Sigma_i y_i b_i$. Then, we get:
$\mathbb{E}(S_3)=\mathbb{E}\|\Phi^*\frac{1}{m}\Sigma_i y_i b_i\|_{K_{\rho}^o} = \mathbb{E}\|\Phi^*L\|_{K_{\rho}^o}$.
Our goal is to bound $\mathbb{E}\|\Phi^*L\|_{K_{\rho}^o}$. Since $y_i$ and $b_i$ are independent random variables (as per Claim \ref{re 4.5}), we can use the law of conditional covariance and the law of iterated expectation. That is, we first condition on $y_i$, and then take expectation with respect to $b_i$.
By conditioning on $y_i$, we have $L\sim\mathcal{N}(0,\beta^2 I_{x^{\perp}})$ where $I_{x^{\perp}} = I - \bar{x}\bar{x}^T$ is the covariance of vector $b_i$ according to claim \ref{re 4.5} and $\beta^2 = \frac{1}{m^2}\Sigma_i y_i^{2}$. Define $g_{x^{\perp}}\sim\mathcal{N}(0,I_{x^{\perp}})$. Therefore, \new{$L\sim \beta g_{x^{\perp}}$ ($L$ equivalent to $g_{x^{\perp}}$ in distribution)}. Putting everything together, we get:
$$\mathbb{E}(S_3)=\mathbb{E}\|\Phi^*L\|_{K_{\rho}^o} = \mathbb{E}\|\Phi^*\beta g_{x^{\perp}}\|_{K_{\rho}^o} = \beta\mathbb{E}\|\Phi^* g_{x^{\perp}}\|_{K_{\rho}^o}.$$
We need to extend the support of distribution of $g_{x^{\perp}}$ and consequently $L$ from $x^{\perp}$ to $\mathbb{R}^n$. This follows from \cite{plan2014high}:
\begin{claim}
Let $g_E$ be a random vector which is distributed as $\mathcal{N}(0,I_E)$. Also, assume that $\Gamma : \mathbb{R}^n\rightarrow\mathbb{R}$ is a convex function. Then, for any subspace $E$ of $\mathbb{R}^n$ such that $E\subseteq F$, we have:
$$\mathbb{E}(\Gamma(g_E))\leq\mathbb{E}(\Gamma(g_F)) .$$
\end{claim}
Hence, we can orthogonally decompose $\mathbb{R}^n$ as $\mathbb{R}^n=D\oplus C$ where $D$ is a subspace supporting $x^{\perp}$ and $C$ is the orthogonal subspace onto it. 
Thus, $g_{\mathbb{R}^n}=g_D + g_C$ in distribution such that $g_D\sim\mathcal{N}(0,I_D), \ g_C\sim\mathcal{N}(0,I_C)$. Also, $\|.\|_{K_{\rho}^o}$ is a convex function since it is a semi-norm. Hence,
\begin{align*}
\mathbb{E}_D\|\Phi^*g_D\|_{K_{\rho}^o} &= \mathbb{E}_D\|\Phi^*g_D+\mathbb{E}_C(g_C)\|_{K_{\rho}^o} \\
&=\mathbb{E}_D\|\mathbb{E}_{C|D}(\Phi^*g_D+g_c)\|_{K_{\rho}^o} \\
&\leq\mathbb{E}_D\mathbb{E}_{C|D}\|\Phi^*(g_D+g_C)\|_{K_{\rho}^o} = \mathbb{E}\|\Phi^*g_{\mathbb{R}^n}\|_{K_{\rho}^o} .
\end{align*}
The first inequality follows from Jensen's inequality, while the second inequality follows from the law of iterated expectation. Therefore, we get:
\begin{align*}
\mathbb{E}\|\Phi^*L\|_{K_{\rho}^o} = \mathbb{E}\|\Phi^*\beta g_{x^{\perp}}\|_{K_{\rho}^o} &=\beta\mathbb{E}\|\Phi^* g_{x^{\perp}} \|_{K_{\rho}^o} \\
&\leq \beta \mathbb{E}\|\Phi^*g_{\mathbb{R}^n}\|_{K_{\rho}^o}= \beta\sup_{u\in (K-K)\cap {\rho}B_n^2} \langle\Phi^* g_{\mathbb{R}^n}, u\rangle = \beta W_{\rho}(K) .
\end{align*}
The last equality follows from the fact that $\Phi^*g_{\mathbb{R}^n}\thicksim\mathcal{N}(0,I)$. The final step is to take an expectation with respect to $y_i$, giving us a bound as
$\mathbb{E}(S_3) = \mathbb{E}\|\Phi^*L\|_{K_{\rho}^o}\leq\mathbb{E}(\beta) W_{\rho}(K) \leq\sqrt{\mathbb{E}(\beta^2)} W_{\rho}(K)$,
where $\beta^2 = \frac{1}{m^2}\sum_{i=1}^m y_i^2$. Hence,
\begin{equation}\label{eq 5.5}
\mathbb{E}(S_3) \leq \frac{\eta}{\sqrt{m}}W_{\rho}(K) \, .
\end{equation}
Putting together the results from \eqref{eq 5.3}, \eqref{eq 5.4}, and \eqref{eq 5.5}:
\begin{align*}
\mathbb{E}(\|\Phi^*\widehat{x}_{\lin}-\mu\alpha{w}\|_{K_{\rho}^o})&\leq \mathbb{E}(S_1)+\mathbb{E}(S_2)+\mathbb{E}(S_3)&\\
&\leq \sqrt{\mathbb{E}(S_1)}+\mathbb{E}(S_2)+\mathbb{E}(S_3)\\ 
&\leq \frac{2\alpha {\rho}\sigma}{\sqrt{m}}\left(\|w\|_2+\|z\|_2\varepsilon\right) + 2\mu\alpha {\rho}\|z\|_2\varepsilon + \frac{\eta}{\sqrt{m}}W_{\rho}(K) .
\end{align*}
Therefore, we obtain:
\begin{align}\label{1MT}
 \mathbb{E}\|\widehat{w}-\mu\alpha{w}\|_2& \leq {\rho} +\frac{2}{\rho}\mathbb{E}(\|\Phi^*\widehat{x}_{lin}-\mu\alpha{w}\|_{K_{\rho}^o})\nonumber \\ 
 &\leq {\rho} + \frac{4\alpha \sigma}{\sqrt{m}}\left(\|w\|_2+\|z\|_2\varepsilon\right) + 4\mu\alpha\|z\|_2\varepsilon + \frac{2\eta}{{\rho}\sqrt{m}}W_{\rho}(K).
\end{align}
Recall that  $\alpha = \frac{1}{\|\Phi {w} + \Psi {z}\|_2}$, hence:
\begin{align*}
\|\Phi {w} + \Psi {z}\|_2^2 &\geq \|\Phi {w}\|_2^2 + \|\Psi {z}\|_2^2 - 2|\langle\Phi{w},\Psi{z}\rangle| \\
&\geq \|w\|_2^2+\|z\|_2^2-2\|w\|_2\|z\|_2\varepsilon\nonumber,\\ 
\end{align*}
\text{or,}\hskip 2.3in$\alpha \leq\frac{1}{\sqrt{\|w\|_2^2+\|z\|_2^2-2\|w\|_2\|z\|_2\varepsilon}}.$
\begin{align*}
\mathbb{E}\|\widehat{w}-\mu\alpha{w}\|_2& \leq {\rho} + \frac{4 \sigma}{\sqrt{m}}\left(\frac{\|w\|_2+\|z\|_2\varepsilon}{\sqrt{\|w\|_2^2+\|z\|_2^2-2\|w\|_2\|z\|_2\varepsilon}}\right) \\
&+ 4\mu\left(\frac{\|z\|_2\varepsilon}{\sqrt{\|w\|_2^2+\|z\|_2^2-2\|w\|_2\|z\|_2\varepsilon}}\right) + \frac{2\eta}{{\rho}\sqrt{m}}W_{\rho}(K).
\end{align*}
Now, let $d = \frac{\|z\|_2}{\|w\|_2}$. Hence:
\begin{align}\label{MaintheOneshot}
\mathbb{E}\|\widehat{w}-\mu\alpha{w}\|_2 &\leq{\rho}+ \frac{4 \sigma}{\sqrt{m}}\left(\frac{1+d\varepsilon}{\sqrt{1+d^2-2d\varepsilon}}\right)+ 4\mu\left(\frac{d}{\sqrt{1+d^2-2d\varepsilon}}\right)\varepsilon + \frac{2\eta}{{\rho}\sqrt{m}}W_{\rho}(K).
\end{align}
It is easy to see that $\frac{1+d\varepsilon}{\sqrt{1+d^2-2d\varepsilon}}\leq 2$ and $\frac{d}{\sqrt{1+d^2-2d\varepsilon}}\leq 2$ provided that $\varepsilon\leq 0.65$ (here, the constant $2$ is just selected for convenience). Now, by plugging these bounds in~\eqref{MaintheOneshot}, we obtain the desired result in Theorem~\ref{thm:main}. However, $K$ is a closed star-shaped set (the set of $s$-sparse signals), and therefore $W_{\rho}(K)={\rho}W_1(K)$~\cite{plan2014high}. Now using~\eqref{MainDecEq}, we obtain:
\begin{align}
\mathbb{E}\|\widehat{x}-\mu\bar{x}\|_2\leq 2{\rho} + \frac{4}{\sqrt{m}}\left(4\sigma + \eta\frac{W_{\rho}(K)}{\rho}\right) +16\mu\varepsilon. \nonumber
\end{align}
We can use Lemma 2.3 in \cite{planRmanrobust} and plug in $W_{\rho}(K)\leq C {\rho}\sqrt{s \log (2n/s)}$. Using the above bound on $\alpha$ and by letting ${\rho}\rightarrow 0$, we get:
\begin{align}
\mathbb{E}\|\widehat{x}-\mu\bar{x}\|_2\leq\frac{4}{\sqrt{m}}\left(4\sigma + C\eta\sqrt{s \log (2n/s)}\right) +16\mu\varepsilon,
\end{align}
where $C> 0$ is an absolute constant. This completes the proof of Corollary~\ref{corr:estx}.
\end{proof}

We now prove the high-probability version of the main theorem. As a precursor, we need a few preliminary definitions and lemmas:

\begin{definition}(Subexponential random variable.)\label{subexp}
A random variable $X$ is subexponential if it satisfies the following relation:
\begin{align*}
\mathbb{E}\exp\left(\frac{c X}{\|X\|_{\psi_1}}\right)\leq 2,
\end{align*}
where $c >0$ is an absolute constant. Here, $\|X\|_{\psi_1}$ denotes the $\psi_1$-norm, defined as follows:
\begin{align*}
\|X\|_{\psi_1}=\sup_{p\geq 1}\frac{1}{p}(\mathbb{E}|X|^p)^{\frac{1}{p}}.
\end{align*}
\end{definition}

We should mention that there are other definitions for subexponential random variables (also for subGaussian defined in Definition~\ref{subgau}). Please see~\cite{vershynin2010introduction} for a detailed treatment.

\begin{lemma}\label{SubgaSubexp}
Let $X$ and $Y$ be two subgaussian random variables. Then, $XY$ is a subexponential random variable.
\end{lemma}

\begin{proof}
According to the definition of the $\psi_2$-norm, we have:
\begin{align}
(\mathbb{E}|XY|^p)^{\frac{1}{p}}=\left(\mathbb{E}|X|^p|Y|^p\right)^{\frac{1}{p}}\leq\left(\left(\mathbb{E}|X|^{2p}\right)^{\frac{1}{2p}}\left(\mathbb{E}|Y|^{2p}\right)^{\frac{1}{2p}}\right)\leq \sqrt{2}p\|X\|_{\psi_2}\|Y\|_{\psi_2},
\end{align}
where the first inequality results from Cauchy-schwarz inequality, and the last inequality is followed by the subgaussian assumption on $X$ and $Y$. This shows that the random variable $XY$ is subexponential random variable according to Definition \ref{subexp}.
\end{proof}

\begin{lemma}(Gaussian concentration inequality) See \cite{vershynin2010introduction,ledoux2013probability}.\label{GaussConcen}
Let $(G_x)_{x\in T}$ be a centered gaussian process indexed by a finite set $T$. Then $\forall t>0$:
\begin{align*}
\mathbb{P}(\sup_{x\in T}G_x\geq \mathbb{E}\sup_{x\in T}G_x+t))\leq \exp\left(-\frac{t^2}{2\sigma^2}\right)
\end{align*} 
where $\sigma^2 = \sup_{x\in T}\mathbb{E}G_x^2 < \infty$.
\end{lemma} 
 
\begin{lemma}(Bernstein-type inequality for random variables)~\cite{vershynin2010introduction}.\label{bern}
Let $X_1,X_2,\ldots,X_n$ be independent sub-exponential random variables with zero-mean. Also, assume that $K = \max_{i}\|X_i\|_{\psi_1}$. Then, for any vector $a\in\mathbb{R}^n$ and every $t\geq 0$, we have:
\begin{align*}
\mathbb{P}(|\Sigma_{i}a_iX_i|\geq t)\leq 2\exp\left(-c\min\left\{\frac{t^2}{K^2\|a\|_2^2},\frac{t}{K\|a\|_\infty}\right\}\right).
\end{align*} 
where $c>0$ is an absolute constant.
\end{lemma}

\begin{proof}[Proof of Theorem~\ref{thm:high}]
We follow the proof given in~\cite{plan2014high}. Let $\beta=\frac{s'}{2\sqrt{m}}$ for $0<s'<\sqrt{m}$ where $m$ denotes the number of measurements. In \eqref{eq}, we saw that $\forall \rho>0$:
\begin{align}\label{mainprob}
\|\widehat{x} - \mu\bar{x}\|_2\leq 2(\rho+\frac{2}{\rho}(S_1+S_2+S_3)).
\end{align}
We attempt to bound each term $S_1, S_2,$ and $S_3$ with high probability, and then use a union bound to obtain the desired result.

For $S_1$, we have:
\begin{align*}\label{mainprob1}
S_1\leq\lvert\frac{1}{m}\Sigma_i(y_i\langle a_i, \bar{x}\rangle-\mu) \rvert\|\Phi^*\bar{x}\|_{K_t^o} .
\end{align*}
We note that $y_i$ is a sub-gaussian random variable (by assumption) and $\langle a_i, \bar{x}\rangle$ is a standard normal random variable. Hence, by Lemma~\ref{SubgaSubexp}, $y_i\langle a_i, \bar{x}\rangle$ is a sub-exponential random variable. Also, $y_i\langle a_i, \bar{x}\rangle$ for $i=1,2,\ldots,m$ are independent sub-exponential random variables that can be centered by subtracting their mean $\mu$. Now, we can apply Lemma~\ref{bern} on $\lvert\frac{1}{m}\Sigma_i(y_i\langle a_i, \bar{x}\rangle-\mu)\rvert$. Therefore:
\begin{align*}
\mathbb{P}(\lvert\frac{1}{m}\Sigma_i(y_i\langle a_i, \bar{x}\rangle-\mu))\rvert\geq \eta\beta)\leq2\exp\left(-\frac{c\beta^2\eta^2m}{\|y_1\|_{\psi_2}^2}\right) .
\end{align*}
Here, $\eta$ and $\mu$ are as defined in~\ref{lemma 4.1}. Using the bound on $\|\Phi^*\bar{x}\|_{K_t^o}$, we have:
\begin{align}
S_1\leq2\alpha {\eta\beta\rho}(\|w\|_2+\|z\|_2\varepsilon)
\end{align}
with probability at least $1-2\exp(-\frac{c\beta^2\eta^2m}{\|y_1\|_{\psi_2}^2})$ where $c>0$ is some constant.

For $S_2$ we have:
\begin{align}\label{mainprob2}
S_2\leq 2\mu\alpha {\rho}\|z\|_2\varepsilon,
\end{align}
with probability $1$ since $S_2$ is a deterministic quantity. 

For $S_3$ we have:
\begin{align*}
S_3\leq\|\Phi^*\frac{1}{m}\Sigma_i y_i b_i\|_{K_\rho^o} .
\end{align*}
To obtain a tail bound for $S_3$, we are using the following:
\begin{align*}
S_3\leq\frac{1}{m}(\Sigma_{i}y_i^2)^{1/2}\|\Phi^*g\|_{K_\rho^o}
\end{align*}
We need to invoke the Bernstein Inequality (Lemma~\ref{bern}) for sub-exponential random variables $(y_i^2-\eta^2)$ for $i=1,2,\ldots,m$ which are zero mean subexponential random variables in order to bound $\frac{1}{m}(\Sigma_{i}y_i^2)^{1/2}$. we have 
$\Big|\frac{1}{m}\Sigma_i(y_i^2-\eta^2)\Big|\leq 3\eta^2$ 
with high probability $1- 2\exp(-\frac{cm\eta^4}{\|y_1\|_{\psi_2}^4})$.

Next, we upper-bound $\|\Phi^*g\|$ (where $g\sim\mathcal{N}(0,I)$) with high probability. Since $\Phi$ is an orthogonal matrix, we have that $\Phi^*g\sim\mathcal{N}(0,I)$. Hence, we can use the Gaussian concentration inequality to bound $\Phi^*g$ as mentioned in Lemma~\ref{GaussConcen}. Putting these pieces together, we have:
\begin{align}\label{mainprob3}
S_3\leq\frac{2\eta}{\sqrt{m}}\left(W_\rho(K) + \rho\beta\sqrt{m}\right) ,
\end{align}     
with probability at least $1- 2\exp(-\frac{cm\eta^4}{\|y_1\|_{\psi_2}^4})-\exp(c\beta^2m)$.
Here, $W_\rho(K)$ denotes the local mean width for the set $K_1$ defining in Definition~\ref{def 3.3}.

Now, combining~\eqref{mainprob},~\eqref{mainprob1},~\eqref{mainprob2}, and~\eqref{mainprob3} together with the union bound, we obtain:
\begin{align}
\|\widehat{x} - \mu\bar{x}\|_2 &\leq \frac{4\alpha {\eta s'}(\|w\|_2+\|z\|_2\varepsilon)}{\sqrt{m}}+ 8\mu\alpha {}\|z\|_2\varepsilon + \frac{C\eta}{\sqrt{m}}\sqrt{s\log\frac{2n}{s}} + 4\frac{\eta s'}{\sqrt{m}},\nonumber \\
&\leq \frac{4\eta}{\sqrt{m}}\left(3s'+C'\sqrt{s\log\frac{2n}{s}}\right) + 16\mu\varepsilon,
\end{align}
with probability at least $1-4\exp(-\frac{cs'^2\eta^4}{\|y_1\|_{\psi_2}^4})$. The second inequality is due to the bounds being used in~\eqref{MaintheOneshot} provided that $\varepsilon\leq 0.65$. Also, $C,C', c> 0$ are absolute constants. Here, we have again used the well-known bound on the local mean width of the set of sparse vectors (for example, see Lemma 2.3 of \cite{planRmanrobust}). This completes the proof.
\end{proof}

\subsection{Analysis of \textsc{DHT}}
\label{AnalysisDHT}


Our analysis of \textsc{DHT} occurs in two stages. First, we define a loss function $F(t)$ that depends on the nonlinear link function $g$ and the measurement matrix $A$. We first assume that $F(t)$ satisfies certain regularity conditions (restricted strong convexity/smoothness), and use this to prove algorithm convergence. 
The proof of Theorem~\ref{mainThConvergence} follows the proof of convergence of the iterative hard thresholding (IHT) algorithm in the linear case~\cite{blumensath2009iterative}, and is more closely related to the work of~\cite{yuan2013gradient} who extended it to the nonlinear setting. Our derivation here differs from these previous works in our specific notion of restricted strong convexity/smoothness, and is relatively more concise. Later, we will prove that the RSS/RSC assumptions on the loss function indeed are valid, given a sufficient number of samples that obey certain measurement models. We assume a variety of measurement models including isotropic row measurements as well as subgaussian measurements. To our knowledge, these derivations of sample complexity are novel. 


First, we state the definitions for restricted strong convexity and restricted strong smoothness, abbreviated as \textit{RSC} and \textit{RSS}. The RSC and RSS was first proposed by~\cite{negahban2009unified,raskutti2010restricted}; also, see~\cite{bahmani2013greedy}.

\begin{definition} \label{defRSCRSS}
A function $f$ satisfies the RSC and RSS conditions if one of the following equivalent definitions is satisfied for all $t_1, t_2$ such that $\|t_1\|_0\leq 2s$ and $\|t_2\|_0\leq 2s$:

\begin{align}\label{d1}
\frac{m_{4s}}{2}\|t_2-t_1\|^2_2\leq f(t_2) - f(t_1) - \langle\nabla f(t_1) , t_2-t_1\rangle\leq\frac{M_{4s}}{2}\|t_2-t_1\|^2_2,
\end{align}
\begin{align}\label{d2}
m_{4s}\|t_2-t_1\|^2_2\leq\langle\nabla f(t_2) - \nabla f(t_1), t_2-t_1\rangle\leq M_{4s}\|t_2-t_1\|^2_2,
\end{align}
\begin{align}\label{d3}
m_{4s}\leq\|\nabla^2_{\xi} f(t)\|\leq M_{4s},
\end{align}
\begin{align}\label{d4}
m_{4s}\|t_2-t_1\|_2\leq\|\nabla_{\xi} f(t_2) - \nabla_{\xi} f(t_1)\|_2\leq M_{4s}\|t_2-t_1\|_2,
\end{align}
where $\xi = \textrm{supp}(t_1)\cup \textrm{supp}(t_2), \ |\xi|\leq 4s$. Moreover, $m_{4s}$ and $M_{4s}$ are called the RSC-constant and RSS-constant, respectively. We note that $\nabla_{\xi}f(t)$ denotes the gradient $f$ restricted to set $\xi$. In addition, $\nabla^2_{\xi} f(t)$ is a $4s\times 4s$ sub-matrix of the Hessian matrix $\nabla^2 f(t)$ comprised of row/column indices indexed by $\xi$.
\end{definition}

\begin{proof}(Equivalence of Eqs. \eqref{d1}, \eqref{d2}, \eqref{d3}, \eqref{d4}).
The proof of above equivalent definitions only needs some elementary arguments and we state them here for completeness.
If we assume that \eqref{d1} is given, then by exchanging $t_1$ and $t_2$ in \eqref{d1}, we have:
\begin{align}\label{d11}
\frac{m_{4s}}{2}\|t_1-t_2\|^2_2\leq f(t_1) - f(t_2) - \langle\nabla f(t_2) , t_1-t_2\rangle\leq\frac{M_{4s}}{2}\|t_1-t_2\|^2_2,
\end{align}
by adding \eqref{d11} with \eqref{d1}, inequality in \eqref{d2} is resulted. Now, assume that \eqref{d2} is given. Then we can set $t_2 = t_1 + \Delta(t_2-t_1)$ in \eqref{d2} and then letting $\Delta\rightarrow 0$ results \eqref{d3} according to the definition of second derivative. Next, if we assume that \eqref{d3} is given, then we can invoke the \textit{mean value theorem}~\cite{mcleod1965mean} for twice-differentiable vector-valued multivariate functions:
\begin{align*}
\nabla_{\xi} f(t_2) - \nabla_{\xi} f(t_1) = \int_0^1P^T_{\xi}\nabla^2 f(c t_2 + (1-c)t_1)(t_2-t_1)dt.
\end{align*} 
where $c > 0$ and $P_{\xi}$ denotes the identity matrix which its columns is restricted to set $\xi$ with $\|\xi\|_0\leq 2s$.
It follows that:
\begin{align*}
\big{\|}\nabla_{\xi} f(t_2) - \nabla_{\xi} f(t_1)\big{\|} & \leq\int_0^1\big{\|}P^T_{\xi}\nabla^2 f(c t_2 + (1-c)t_1)(t_2-t_1)\big{\|}dt\\
& \leq M_{4s}\|(t_2-t_1)\|.
\end{align*} 
where the last inequality follows by \eqref{d3}. Similarly, we can establish the lower bound in \eqref{d4} by invoking the Cauchy Schwartz inequality in \eqref{d2}. 

Finally, suppose that \eqref{d3} holds. We can establish \eqref{d1} by performing a Taylor expansion of $f(t)$.
For upper bound in \eqref{d1} and some $0\leq c\leq 1$, we have:
\begin{align*}
f(t_2) & \leq f(t_1) - \langle\nabla f(t_1) , t_2-t_1\rangle + \frac{1}{2}\left(t_2-t_1\right)^T\nabla^2_{\xi}
f(ct_2+(1-c)t_1)\left(t_2-t_1\right) \\
 & \leq f(t_1) - \langle\nabla f(t_1) , t_2-t_1\rangle + \frac{M_{4s}}{2}\|t_1-t_2\|^2_2.
\end{align*}
The lower bound in \eqref{d1} also follows similarly.
\end{proof}

We now give a proof that \textsc{DHT} enjoys the linear convergence, as stated in Theorem~\ref{mainThConvergence}. Recall that as opposed to the commonly used least-squares loss function, we instead define a special objective function:
\begin{align*}
F(t) = \frac{1}{m}\sum_{i=1}^m \Theta(a_i^T\Gamma t) - y_i a_i^T\Gamma t,
\end{align*}
where $\Gamma = [\Phi \ \Psi]$, $t = [ w;z]\in\mathbb{R}^{2n}$, and $\Theta'(x) = g(x)$. The gradient and Hessian of the objective function are given as follows:
\begin{align}\label{Gradient}
\nabla F(t) = \frac{1}{m}\sum_{i=1}^{m}\Gamma^T a_i g(a_i^T\Gamma t) - y_i\Gamma^Ta_i \, ,
\end{align}
\begin{align}\label{Hessian}
\nabla^2 F(t) = \frac{1}{m}\sum_{i=1}^{m}\Gamma^T a_i a_i^T\Gamma g^\prime(a_i^T\Gamma t) \, .
\end{align}

We start with the projection step in Algorithm \ref{algDHT}. In what follows, the superscript $k$ denotes the $k$-th iteration. Let $t^{k+1} = [t_1^k;t_2^k]\in\mathbb{R}^{2n}$ be the constituent vector as the $k^{\textrm{th}}$ iteration. Hence,
$$
t^{k+1} =\mathcal{P}_{2s}\left(t^k - \eta^{\prime}\nabla F(t^k)\right),
$$
where $\eta^{\prime}$ denotes the step size in Algorithm \ref{algDHT} and $\mathcal{P}_{2s}(.)$ denotes the hard thresholding operation. Furthermore, $\nabla F(t^k)$ is the gradient of the objective function at iteration $k$. Moreover, we define sets $S^k, S^{k+1}, S^*$ as follows, each of whose cardinalities is no greater than $2s$:
\[
\text{supp}(t^k) = S^{k}, \ \text{supp}(t^{k+1}) = S^{k+1}, \ \text{supp}(t^*) = S^{*}. 
\]
Moreover, define $S^{k}\cup S^{k+1}\cup S^{*} = J_k = J$ such that $|J|\leq 6s$.

Define $b_k = t^k - \eta^{\prime}\nabla_J F(t^k)$. Then,
\begin{align}\label{proj}
\|t^{k+1} - t^*\|_2\leq\|t^{k+1} - b\|_2 + \|b - t^{*}\|_2\leq2\|b-t^*\|_2,
\end{align}
where $t^{*} = [t_1^*;t_2^*]\in\mathbb{R}^{2n}$ such that $\|t^*\|_0\leq2s$ is the solution of the optimization problem in~\eqref{optprob}. The last inequality follows since $t^{k+1}$ is generated by taking the $2s$ largest entries of $t^k - \eta' \nabla F(t^k)$; by definition of $J$ (Please note that $J$ depends
on $k$, i.e., $J := J_k$), $t^{k+1}$ also has the minimum Euclidean distance to $b_k$ over all vectors with cardinality $2s$. Moreover:
\begin{align}\label{linearconv}
\|b_k-t^*\|_2 & =\|t^k - \eta^{\prime}\nabla_J F(t^k) - t^*\|_2 \nonumber \\
& \leq\|t^k - t^* - \eta^{\prime}\left(\nabla_J F(t^k)-\nabla_J F(t^*)\right)\|_2 + \eta^{\prime}\|\nabla_J F(t^*)\|_2.
\end{align}
Now, by invoking RSC and RSS in the Definition~\ref{defRSCRSS}, we have:
\begin{align*}
\|t^k - t^* - \eta^{\prime}\left(\nabla_J F(t^k)-\nabla_J F(t^*)\right)\|_2^2\leq(1+{\eta^{\prime}}^2M_{6s}^2-2\eta^{\prime} m_{6s})\|t^k-t^*\|_2^2,
\end{align*}
where $M_{6s}$ and $m_{6s}$ denote the RSC and RSS constants. The above inequality follows by the upper bound of \eqref{d4} and the lower bound of \eqref{d2} in Definition~\ref{defRSCRSS} with the restriction set $\xi$ chosen as $J$. Now let $q =\sqrt{1+{\eta^{\prime}}^2M_{6s}^2-2\eta^{\prime} m_{6s}}$. By~\eqref{proj} and~\eqref{linearconv}, we have:
\begin{align}\label{mainlinearconv}
\|t^{k+1} - t^*\|_2\leq2q\|t^k-t^*\|_2 + 2\eta^{\prime}\|\nabla_J F(t^*)\|_2. 
\end{align}
In order for the algorithm to exhibit linear convergence, we need to have $2q< 1$. That is, 
$${\eta^{\prime}}^2M_{6s}^2-2\eta^{\prime} m_{6s} + \frac{3}{4}<0 .$$ 
By solving this quadratic inequality with respect to $\eta'$, we obtain that $\eta^{\prime}$, $m_{6s}$, and $M_{6s}$ should satisfy 
$$1\leq\frac{M_{6s}}{m_{6s}}\leq\frac{2}{\sqrt{3}}, \ \ \ \frac{0.5}{M_{6s}}<\eta^{\prime}<\frac{1.5}{m_{6s}} .$$ 
The bound in~\eqref{mainlinearconv} shows that after enough iterations the first term vanishes and the quality of estimation depends on the vanishing speed of the second term, $2\eta^{\prime}\|\nabla_J F(t^*)\|_2 $ that is determined by the number of measurements.
To bound the gradient in second term, $\|\nabla_J F(t^*)\|_2$, we need the following lemma:
\begin{lemma}(Khintchine inequality~\cite{vershynin2010introduction}.)
Let $X_i$ be a finite number of independent and zero mean subgaussian random variables with unit variance. Assume that $\|X_i\|_{\psi_2}\leq r$. Then, for any real $b_i$ and $p\geq2$:
\begin{align*}
\left(\sum_{i}b_i^2\right)^{\frac{1}{2}}\leq\left(\mathbb{E}|\sum_{i}b_iX_i|^p\right)^\frac{1}{p}\leq Cr\sqrt{p}\left(\sum_ib_i^2\right)^{\frac{1}{2}}.
\end{align*}
\end{lemma} 
Recall that our measurement model is given by:
$$y_i = g(a_i^T\Gamma t) + e_i,  \ \ i=1,\ldots,m.$$
As mentioned above, we assume that $e_i$ represents the  additive subgaussian noise with $\|e_i\|_{\psi_2}\leq\tau$ for $i=1 \ldots m$. 

We leverage the Khintchine inequality to bound $\mathbb{E}\|\nabla_J F(t^*)\|_2$ under the subgaussian assumption on $e_i$.
Denoting by $\left(\nabla_J F(t^*)\right)_k$ as the $k^{\text{th}}$ entry of the gradient (restricted to set $J$), from the Khintchine inequality, and for each $k=1, \ldots, |J|$, we have:

\begin{align}\label{Khintchinbound}
\left(\mathbb{E}\left|\left(\nabla_J F(t^*)\right)_k\right|^2 \right)^\frac{1}{2} & \overset{r_1}{=} \left(\mathbb{E}\left(\frac{1}{m}\sum_{i=1}^{m}\left(\Gamma_J\right)_k^T a_i e_i\right)^2\right)^\frac{1}{2} \nonumber \\
& \overset{r_2}{\leq}\frac{1}{m}\mathbb{E}\left(C\tau\sqrt{2}\left(\sum_{i=1}^{m}\left(\left(\Gamma_J\right)_k^T a_i\right)^2 \right)^\frac{1}{2}\right)  \nonumber\\
& \overset{}{\leq}\frac{1}{m}C\tau\sqrt{2}\left(\sum_{i=1}^m\left(\Gamma_J\right)_k^T \mathbb{E}\left(a_ia_i^T\right)\left(\Gamma_J\right)_k\right)^\frac{1}{2} \nonumber \\
& \overset{r_3}{=} \frac{C\tau\sqrt{2}}{\sqrt{m}},
\end{align} 
where $\Gamma_J$ denotes the restriction of the columns of the dictionary to set $J$ with $|J|\leq 6s$ such that $3s$ of the columns are selected from each basis of the dictionary. Here, $r_1$ follows from~\eqref{Gradient}, $r_2$ follows from the Khintchine inequality with $p=2$ and the fact that $e_i$ are independent from $a_i$. Finally, $r_3$ holds since the rows of $A$ are assumed to be isotropic random vectors.
Now, we can bound $\mathbb{E}\|\nabla_J F(t^*)\|_2$ as follows:
\begin{align}\label{finalKhintchinbound}
\mathbb{E}\|\nabla_J F(t^*)\|_2\leq\sqrt{\mathbb{E}\|\nabla_J F(t^k)\|_2^2}\leq C^{\prime}\tau\sqrt{\frac{s}{m}},
\end{align}
where $C^{\prime} > 0$ is an absolute constant and the last inequality is followed by~\eqref{Khintchinbound} and the fact that $\|J\|_0\leq 6s$.

\begin{proof}[Proof of Theorem~\ref{mainThConvergence}]

By using induction on~\eqref{mainlinearconv}, taking expectations, and finally using the bound stated in~\eqref{finalKhintchinbound}, we obtain the desired bound in Theorem~\ref{mainThConvergence} as follows:

\begin{align}\label{LinearConvergence}
\|t^{k+1} - t^*\|_2 & \leq\left(2q\right)^k\|t^0-t^*\|_2 + \frac{2\eta^{\prime}}{1-2q}\|\nabla_J F(t^*)\|_2 \nonumber \\
& \leq\left(2q\right)^k\mathbb{E}\|t^0-t^*\|_2 + C\tau\sqrt{\frac{s}{m}},
\end{align}
where $C>0$ is a constant which depends only on the step size, $\eta^{\prime}$ and $q$. Also, $t^0$ denotes the initial value for the constituent vector, $t$. In addition, in the noiseless case ($\tau = 0$), if we denote $\kappa$ as the desired accuracy for solving optimization problem~\eqref{optprob}, then the number of iterations to achieve the accuracy $\kappa$ is given by $N = \mathcal{O}(\log\frac{\|t^0-t^*\|_2}{\kappa})$. 
\end{proof}

\label{sec::proofRsscRss}
In the above convergence analysis of \textsc{DHT}, we assumed that objective function in~\eqref{optprob}, $F(t)$ satisfies the RSC/RSS conditions. In this section, we validate this assumption via the proofs for Theorems~\ref{SampleComplexityNonormal} and \ref{SampleComplexitysubg}.
%
%
As discussed above, we separately analyze two cases.

\subsubsection{Case (a): isotropic rows of $A$}
We first consider the case where the rows of the measurement matrix $A$ are sampled from an isotropic probability distribution in $\R^n$. 
Specifically, we make the following assumptions on $A$:
\begin{enumerate}
\item the rows of $A$ are independent isotropic vectors. That is, $\mathbb{E} a_i a_i^T = I_{n\times n}$ for $i=1\dots m$.
\item $\|a_i^T\Gamma_{\xi}\|_{\infty}\leq\vartheta$  for $i=1\dots m$.
\end{enumerate}

\begin{remark}
Assumption $2$ is unavoidable in our analysis, and indeed this is one of the cases where our derivation differs from existing proofs. The condition $||a_i^T\Gamma_{\xi}||_{\infty}\leq\vartheta$ requires that all entries in $A\Gamma_{\xi}$ are bounded by some number $\vartheta$. In other words, $\vartheta$ captures the cross-coherence between the measurement matrix, $A$ and  the dictionary $\Gamma_{\xi}=[\Phi \ \Psi]_{\xi}$ and controls the interaction between these two matrices. 
Without this assumption, one can construct a counter-example with the Hessian of the objective to be zero with high probability (for instance, consider partial DFT matrix as the measurement matrix $A$ and $\Gamma_{\xi} = [I \ \Psi]_{\xi}$ with $\Psi$ being the inverse DFT basis). 
\end{remark}

Modifying~\eqref{Hessian}, we define the \emph{restricted} Hessian matrix as a $4s\times 4s$ sub-matrix of the Hessian matrix:
\begin{align}\label{ReCher}
\nabla^2_{\xi} F(t) = \frac{1}{m}\sum_{i=1}^{m}\Gamma^T_{\xi} a_i a_i^T\Gamma_{\xi} g^\prime(a_i^T\Gamma t),\quad \|\xi\|_0\leq 4s.
\end{align}

Here, $\Gamma_{\xi}$ is the restriction of the columns of the dictionary $\Gamma = [\Phi  \ \Psi]$ with respect to set $\xi$, such that $2s$ columns are selected from each basis. Let $S_i =\Gamma^T_{\xi} a_i a_i^T\Gamma_{\xi} g^\prime(a_i^T\Gamma t), i=1\dots m$. As per our assumption in Section~\ref{Sec::Perm}, the derivative of the link function, $g(x)$ satisfies $0 < l_1\leq g^\prime(x)\leq l_2$. By this assumption, it is guaranteed that $\lambda_{\min}(S_i)\geq 0, i=1\dots m$; this follows since $\Gamma^T_{\xi} a_i a_i^T\Gamma_{\xi}$ is a positive semidefinite matrix and $g^{\prime}>0$, we have $\lambda_{\min}(S_i)=\lambda_{\min}(\Gamma^T_{\xi} a_i a_i^T\Gamma_{\xi})g^{\prime}\geq 0$. 

Let $\Lambda_{\max} = \max\limits_{\xi}\lambda_{\max}(\nabla^2_{\xi}F(t))$ and $\Lambda_{\min} = \min\limits_{\xi}\lambda_{\min}(\nabla^2_{\xi}F(t))$ where $\lambda_{\min}$ and $\lambda_{\max}$ denote the minimum and maximum eigenvalues of the restricted Hessian matrix. Furthermore, let $U$ be any index set with $|U|\leq 6s$ such that $\xi\subseteq U$. We have:

\begin{align*}
l_1\min_{U}\lambda_{\min}\left(\frac{1}{m}\sum_{i=1}^{m}\Gamma^T_{U} a_i a_i^T\Gamma_{U}\right)\leq\Lambda_{\min}\leq\Lambda_{\max}\leq l_2\max_{U}\lambda_{\max}\left(\frac{1}{m}\sum_{i=1}^{m}\Gamma^T_{U} a_i a_i^T\Gamma_{U}\right).
\end{align*}
Here, $\Gamma_{U}$ is the restriction of the columns of $\Gamma$ with respect to a set $U$ such that $3s$ columns is selected from each basis. By taking expectations, we obtain:

\begin{align}\label{boundonHessianr}
l_1\mathbb{E}\min_{U}\lambda_{\min}\left(\frac{1}{m}\sum_{i=1}^{m}\Gamma^T_{U} a_i a_i^T\Gamma_{U}\right)\leq\mathbb{E}\Lambda_{\min}\leq\mathbb{E}\Lambda_{\max}\leq l_2\mathbb{E}\max_{U}\lambda_{\max}\left(\frac{1}{m}\sum_{i=1}^{m}\Gamma^T_{U} a_i a_i^T\Gamma_{U}\right).
\end{align}
Inequality in~\eqref{boundonHessianr} shows that for proving RSC and RSS, we need to bound the expectation of the maximum and minimum eigenvalues of $\frac{1}{m}\sum_{i=1}^{m}\Gamma^T_{\xi} a_i a_i^T\Gamma_{U}$ over sets $U$ with $|U| \leq 6s$. We should mention that \eqref{boundonHessianr} establishes RSC/RSS constants in expectation. One can establish RSC/RSS in tail probability using results in~\cite{rudelson2008sparse,ledoux2013probability}. 

As our main tool for bounding the RSC/RSS constants, we use the \textit{uniform Rudelson's inequality}~\cite{vershynin2010introduction,rudelson2008sparse}. 

\begin{lemma}(Uniform Rudelson's inequality)
Let $x_i$ be vectors in $\mathbb{R}^n$ for $i=1,\ldots,m$ and $m\leq n$. Also assume that the entries of $x_i$'s are bounded by $\vartheta$, that is, $\|x_i\|_{\infty}\leq \vartheta$. Let $h_i$ denote independent Bernoulli random variables (with parameter 1/2) for $i=1\dots m$. Then for every set $\Omega\subseteq [n]$, we have:
\begin{align}\label{rudel}
\mathbb{E}\max_{|\Omega|\leq n}\Big{\|}\sum_{i=1}^{m} h_i(x_i)_{\Omega}(x_i)^T_{\Omega}\Big{\|}\leq C_{\vartheta}l\sqrt{|\Omega|}\max_{|\Omega|\leq n}\Big{\|}\sum_{i=1}^{m}(x_i)_{\Omega}(x_i)^T_{\Omega}\Big{\|}^{\frac{1}{2}},
\end{align}
where $(x_i)_{\Omega}$ denotes the restriction of $x_i$ to $\Omega$, $l = \log(|\Omega|)\sqrt{\log m}\sqrt{\log n}$, and $C_{\vartheta}$ denotes the dependency of $C$ only on $\vartheta$.
\end{lemma} 

Before using the above result, we need to restate the uniform version of the standard symmetrization technique (Lemma $5.70$ in~\cite{vershynin2010introduction}):

\begin{lemma}(Uniform symmetrization)
\label{unisym}
Let $x_{ik}$, $i=1\dots m$ be independent random vectors in some Banach space where indexed by some set $\Xi$ such that $k\in\Xi$. Also, assume that $h_i$, $i=1\dots m$ denote independent Bernoulli random variables (with parameter 1/2) for $i=1\dots m$. Then, 
\begin{align}
\mathbb{E}\sup_{k\in\Xi}\Big{\|}\sum_i^m(x_{ik}-\mathbb{E}x_{ik})\Big{\|}\leq 2\mathbb{E}\sup_{k\in\Xi}\Big{\|}\sum_i^m h_ix_{ik}\Big{\|}.
\end{align}
\end{lemma}

Now we apply the Uniform Rudelson's inequality on $\lambda_{\max}\left(\frac{1}{m}\sum_{i=1}^{m}\Gamma^T_{U} a_i a_i^T\Gamma_{U}\right)$ over all set $U$ with $|U|\leq 6s$. We have:
\begin{align}\label{mainrudelsymm}
R\overset{\Delta}{=}\mathbb{E}\max_{U}\big{\|}\frac{1}{m}\sum_{i=1}^{m}\Gamma^T_{U} a_i a_i^T\Gamma_{U}-\Gamma^T_{U}\Gamma_{U}\Big{\|}& \overset{r_1}{\leq}2\mathbb{E}\max_{U}\Big{\|}\frac{1}{m}\sum_{i=1}^{m}h_i\Gamma^T_{U} a_i a_i^T\Gamma_{U}\Big{\|} \nonumber \\
& \overset{r_2}{\leq} \frac{C_{\vartheta}l\sqrt{6s}}{\sqrt{m}}\mathbb{E}\max_{U}\Big{\|}\new{\frac{1}{m}}\sum_{i=1}^{m}\Gamma^T_{U} a_i a_i^T\Gamma_{U}\Big{\|}^{\frac{1}{2}},
\end{align}
where $r_1$ follows from Lemma~\ref{unisym} with $h_i$ defined in this lemma and $r_2$ follows from \eqref{rudel}. In addition $l = \log(6s)\sqrt{\log m}\sqrt{\log 2n}$. Then by application of a triangle inequality, we have: 
\begin{align*}
\mathbb{E}\max_{U}\Big{\|}\new{\frac{1}{m}}\sum_{i=1}^{m}\Gamma^T_{U} a_i a_i^T\Gamma_{U}\Big{\|}\leq R+\max_{U}\big{\|}\Gamma^T_{U}\Gamma_{U}\big{\|}.
\end{align*}
On the other hand by Cauchy-Schwarz inequality, we get:
\begin{align*}
\mathbb{E}\max_{U}\Big{\|}\new{\frac{1}{m}}\sum_{i=1}^{m}\Gamma^T_{U} a_i a_i^T\Gamma_{U}\Big{\|}^{\frac{1}{2}}\leq \left(\mathbb{E}\max_{U}\Big{\|}\new{\frac{1}{m}}\sum_{i=1}^{m}\Gamma^T_{U} a_i a_i^T\Gamma_{U}\Big{\|} \right)^{\frac{1}{2}}
\end{align*}
By combining the above inequalities, we obtain:
\begin{align}\label{quadra}
R\leq\frac{C'_{\vartheta}l\sqrt{s}}{\sqrt{m}}\left(R+\max_{U}\big{\|}\Gamma^T_{U}\Gamma_{U}\big{\|}\right)^{\frac{1}{2}},
\end{align}
where $C'_{\vartheta}$ depends only on $\vartheta$. This inequality is a quadratic inequality in terms of $R$ and is easy to solve. By noting $\beta =\max_{U}\big{\|}\Gamma^T_{U}\Gamma_{U}\big{\|}$, we can write ~\eqref{quadra} as $\frac{R}{\beta}\leq\frac{C'_{\vartheta}l\sqrt{s}}{\sqrt{m}}\frac{1}{\beta}\left(1+\frac{R}{\beta}\right)^{\frac{1}{2}}$. Now we can consider two cases; either $\frac{R}{\beta}\leq 1$, or $\frac{R}{\beta}> 1$. As a result, we have:
\begin{align}\label{boundforsamp}
R\leq\max\left(\delta_0\left(\max_{U}\big{\|}\Gamma^T_{U}\Gamma_{U}\big{\|}\right)^{\frac{1}{2}},\delta_0^2\right),
\end{align}
where $\delta_0 = \frac{C'_{\vartheta}l\sqrt{s}}{\sqrt{m}}$. In addition, we can use the {Gershgorin Circle Theorem} \cite{horn2012matrix} to bound $\lambda_{\max}(\Gamma^T_{U}\Gamma_{U}) = \|\Gamma^T_{U}\Gamma_{U} \|$ and $\lambda_{\min}(\Gamma^T_{U}\Gamma_{U})$. This follows since: 
$$\Gamma^T_{U}\Gamma_{U} = \begin{bmatrix} I\ \ \Phi^T\Psi\\ \Psi^T\Phi \ \ I \end{bmatrix}_{6s\times 6s},$$
and hence we have:
\begin{align*}
\Big|\lambda_i(\Gamma^T_{U}\Gamma_{U}) - 1\Big|\leq (6s-1)\gamma,\quad i=1\dots 6s,
\end{align*}
where $\gamma$ denotes the mutual coherence of $\Gamma$. Hence, the following holds for all index set $U$:
\begin{align}\label{lambdaMinMaxRudel}
1-(6s-1)\gamma\leq\lambda_{\min}(\Gamma^T_{U}\Gamma_{U})\leq\lambda_{\max}(\Gamma^T_{U}\Gamma_{U})\leq 1+(6s-1)\gamma,
\end{align}
provided that $\gamma\leq\frac{1}{6s-1}$ to have nontrivial lower bound. 

\begin{proof}[Proof of Theorem~\ref{SampleComplexityNonormal}]
If we choose $m\geq\left(\frac{C''_{\vartheta}}{\delta^2}s\log(n)\log^2 s\log\left(\frac{1}{\delta^2}s\log(n)\log^2 s\right)\left(1+(6s-1)\gamma\right)\right)$ in \eqref{boundforsamp}, then we have $R\leq\delta$ for some $\delta\in(0,1)$ and $C''_{\vartheta}>0$ which depends only on $\vartheta$. If $s = o(1/\gamma) $, then we obtain the stated sample complexity in Theorem~\ref{SampleComplexityNonormal}.
\end{proof}

\subsubsection{Case (b): isotropic subgaussian rows of $A$}
\label{sub::subg}
Now, suppose that the measurement matrix $A$ has independent isotropic subgaussian rows. We show that under this assumption, one can obtain better sample complexity bounds compared to the previous case. We use the following argument (which is more or less standard; see~\cite{rauhut2010,mendelson2008,candesneedell2010}).
%
%
Let $\Gamma = [\Phi \ \Psi]$, and let $B_U = A \Gamma_U$ for any fixed $|U| \leq 6s$, where $3s$ elements are chosen from each basis. According to the notation from Section~\ref{sec::proofRsscRss}, we have:
\begin{align}\label{boundonHessian}
l_1\min_{U}\lambda_{\min}\left(\frac{1}{m}B_{U}^TB_{U}\right)\leq\Lambda_{\min}\leq\Lambda_{\max}\leq l_2\max_{U}\lambda_{\max}\left(\frac{1}{m}B_{U}^TB_{U}\right).
\end{align}
where $l_1, l_2$ are upper and lower bounds on the derivative of the link function.
Therefore, all we need to do is to bound the maximum and minimum singular values of $\frac{1}{\sqrt{m}}B_{U}$. To do so, we use the fact that if the rows of $A$ are $m$ independent copies of an isotropic vector with bounded $\psi_2$ norm, then the following holds for any fixed vector $v \in \R^{2n}$:
\begin{align}\label{deltta}
\Big|\frac{1}{m}\Big{\|}B v\Big{\|}_2^2 - \Big{\|}\Gamma v\Big{\|}_2^2\Big|\leq\max(\delta_{0},\delta_{0}^2) \overset{\Delta}{=}\varepsilon^{\prime},
\end{align}
with high probability where $\delta_{0} = C\sqrt{\frac{6s}{m}}+\frac{t}{\sqrt{m}}$ for some absolute constant $C>0$~\cite{mendelson2008} and $\forall t>0$. Now fix any set $U$ as above. Then, one can show using a covering number argument (for example, Lemma 2.1 in~\cite{rauhut2010}) with $\frac{1}{4}$-net ($\mathcal{N}_{\frac{1}{4}}$) of the unit sphere and applying the upper bound in~\eqref{lambdaMinMaxRudel}  for any $v \in U$, we get:
\begin{align*}
\mathbb{P}\left(\Big|\frac{1}{m}\big{\|}B_{U}v\big{\|}_2^2 - \big{\|}\Gamma_{U}v\big{\|}_2^2\Big|\geq\frac{\varepsilon^{\prime}}{2}\right)& \leq 2(9)^{6s}\exp\left(-\frac{c}{\left(1+(6s-1)\gamma\right)^2}\delta_{0}^2m\right)
\end{align*}
where $c>0$ is a constant. Taking a union bound over all possible subsets $U$ with $|U| \leq 6s$ and choosing $t = C_1\left(1+(6s-1)\gamma\right)^2\sqrt{s\log\frac{en}{6s}} + u\sqrt{m}$ in $\delta_0$ where $C_1>0$ (absolute constant) and $u >0$ are arbitrary small constants, we obtain:
\begin{align*}
\mathbb{P}\left(\max_{U} \ \max_{v\in\mathcal{N}_{\frac{1}{4}}}\Big|\frac{1}{m}\big{\|}Bv\big{\|}_2^2 - \big{\|}\Gamma_{U}v\big{\|}_2^2\Big|\geq\frac{\varepsilon^{\prime}}{2}\right)\leq 2\exp\left(-c_2u^2m\right),
\end{align*}
where $c_2>0$ is an absolute constant. By plugging $t$ in the expression of $\delta_0$ and letting $u\leq\delta/2$ and $m$ sufficiently large, we have $\delta_0\leq\delta$ for some $\delta\in\left(0,1\right)$. 
As a result, from \eqref{deltta}, we have:
\begin{align}\label{eigen}
\max_{U}\Big{\|}\frac{1}{m} B_{U}^TB_{U} - \Gamma_{U}^T\Gamma_{U}\Big{\|}\leq\delta,
\end{align}
with probability at least $1-2\exp\left(-c_2u^2m\right)$. Therefore, for sufficiently large $m$ (that we specify below), the following holds with high probability:
\begin{align*}
\lambda_{\min}\left(\Gamma_{U}^T\Gamma_{U}\right)-\delta\leq\lambda_{\min}\left(\frac{1}{m}B_{U}^TB_{U}\right)\leq\lambda_{\max}\left(\frac{1}{m}B_{U}^TB_{U}\right)\leq\lambda_{\max}\left(\Gamma_{U}^T\Gamma_{U}\right)+\delta
\end{align*}
We use~\eqref{lambdaMinMaxRudel} to bound $\lambda_{\max}(\Gamma^T_{U}\Gamma_{U}) = \|\Gamma^T_{U}\Gamma_{U} \|$ and $\lambda_{\min}(\Gamma^T_{U}\Gamma_{U})$; as a result, 
\begin{align}\label{finalequall}
1-(6s-1)\gamma-\delta\leq\lambda_{\min}\left(\frac{1}{m}B_{U}^TB_{U}\right)\leq\lambda_{\max}\left(\frac{1}{m}B_{U}^TB_{U}\right)\leq1+(6s-1)\gamma+\delta
\end{align}
Thus, we obtain the desired bound in~\eqref{boundonHessian}. That is:
\begin{align}
l_1\left(1-(6s-1)\gamma-\delta\right)\leq\Lambda_{\min}\leq\Lambda_{\max}\leq l_2\left(1+(6s-1)\gamma-\delta\right).
\end{align}
holds with high probability for some $0<\delta< 1-(6s-1)\gamma$. 

\begin{proof}[Proof of Theorem~\ref{SampleComplexitysubg}]
The probability of failure of the above statement can be vanishingly small if we set $m\geq \frac{C'}{\delta^2} s\log\frac{n}{s}$ for some $\delta\in\left(0,1\right)$ and absolute constant $C'>0$. Note that we only obtain nontrivial upper and lower bounds on $\Lambda_{\min}, \Lambda_{\max}$ if $\gamma\leq\frac{1}{6s-1}$. Assuming constant $\delta$ and coherence $\gamma$ inversely proportional to $s$, we obtain the required sample complexity of DHT as:
$m=\mathcal{O}\left(s\log\frac{n}{s}\right)$. 
\end{proof}

For both cases (a) and (b), RSC and RSS constants follow by setting $M_{6s} \leq l_2\left(1+(6s-1)\gamma+\delta\right)$ and $m_{6s} \geq l_1\left(1-(6s-1)\gamma-\delta\right)$.
As we discussed in the begging of section~\ref{AnalysisDHT}, we require that $\frac{0.5}{M_{6s}}<\eta^{\prime}<\frac{1.5}{m_{6s}}$ in order to establish linear convergence of \textit{DHT}. Hence, for linear convergence, the step size must satisfy: 
\begin{align*}
\frac{0.5}{l_2\left(1+(6s-1)\gamma+\delta\right)}<\eta^{\prime}<\frac{1.5}{l_1\left(1-(6s-1)\gamma-\delta\right)}
\end{align*}
for some $0<\delta< 1-(6s-1)\gamma$.

\section{Conclusion}
\label{Sec::Con}

In this paper, we consider the problem of demixing sparse signals from their nonlinear measurements. We specifically study the more challenging scenario where only a limited number of nonlinear measurements of the superposition signal are available. As our primary contribution, we propose two fast algorithms for recovery of the constituent signals, and support these algorithms with the rigorous theoretical analysis to derive nearly-tight upper bounds on their sample complexity for achieving stable demixing. 

%

We anticipate that the problem of demixing signals from nonlinear observations can be used in several different practical applications. As future work, we intend to extend our methods to more general signal models (including rank-sparsity models for matrix valued data), as well as robust recovery under more general nonlinear observation models.


\bibliographystyle{unsrt}
\bibliography{../Common/chinbiblio,../Common/csbib,../Common/mrsbiblio}

\end{document}